\documentclass[letterpaper]{article} % DO NOT CHANGE THIS
\usepackage{aaai2026}  % DO NOT CHANGE THIS
\usepackage{times}  % DO NOT CHANGE THIS
\usepackage{helvet}  % DO NOT CHANGE THIS
\usepackage{courier}  % DO NOT CHANGE THIS
\usepackage[hyphens]{url}  % DO NOT CHANGE THIS
\usepackage{graphicx} % DO NOT CHANGE THIS
\usepackage{natbib}  % DO NOT CHANGE THIS AND DO NOT ADD ANY OPTIONS TO IT
\usepackage{caption} % DO NOT CHANGE THIS AND DO NOT ADD ANY OPTIONS TO IT
\usepackage{algorithm}
\usepackage{algorithmic}

\usepackage{booktabs}
\usepackage{bm}
\usepackage{amsmath}
\usepackage{amssymb}
\usepackage{mathtools}

\usepackage{amsthm}

\usepackage[capitalise,nameinlink,noabbrev]{cleveref}
\theoremstyle{plain}
\newtheorem{theorem}{Theorem}
\newtheorem{lemma}{Lemma}

\theoremstyle{definition}

\theoremstyle{remark}

\usepackage{graphicx}

\usepackage{subcaption}
\usepackage{tikz}
\usepackage{pgfplots}
\pgfplotsset{compat=1.18}
\usepackage{algorithm}
\usepackage[lined,boxed,commentsnumbered,algo2e]{algorithm2e}
\usepackage{multirow}
\usepackage{multicol}

\usepackage{dsfont}   % For \mathds
\usepackage[colorinlistoftodos]{todonotes}

\usepackage{newfloat}
\usepackage{listings}
\DeclareCaptionStyle{ruled}{labelfont=normalfont,labelsep=colon,strut=off} % DO NOT CHANGE THIS
\floatstyle{ruled}
\newfloat{listing}{tb}{lst}{}
\floatname{listing}{Listing}
\title{RoS-Guard: Robust and Scalable Online Change Detection with Delay-Optimal Guarantees}
\author{
    Zelin Zhu\equalcontrib, Yancheng Huang\equalcontrib, Kai Yang\textsuperscript{\rm }\thanks{Corresponding author}\\
}
\affiliations{
    \textsuperscript{\rm }School of Computer Science and Technology, Tongji University, China\\
    4800 Caoan Road, Jiading District, Shanghai, China\\
    zelinzhu@tongji.edu.cn, yanchenghuang@tongji.edu.cn, kaiyang@tongji.edu.cn
}

\usepackage{bibentry}
\begin{document}

\maketitle

\begin{abstract}
Online change detection (OCD) aims to rapidly identify change points in streaming data and is critical in applications such as power system monitoring, wireless network sensing, and financial anomaly detection. Existing OCD methods typically assume precise system knowledge, which is unrealistic due to estimation errors and environmental variations. Moreover, existing OCD problem optimization often struggle with efficiency in large-scale systems. To overcome these challenges, we propose RoS-Guard, a robust and optimal OCD algorithm tailored for linear systems with uncertainty. Through a tight relaxation and reformulation of the OCD optimization problem, RoS-Guard employs neural unrolling to enable efficient parallel computation via GPU acceleration. The algorithm provides theoretical guarantees on performance, including expected false alarm rate and worst-case average detection delay. Extensive experiments validate the effectiveness of RoS-Guard and demonstrate significant computational speedup in large-scale system scenarios.
\end{abstract}


\section{Introduction}

Online Change Detection (OCD) aims to detect distributional changes in a data stream as new observations arrive sequentially, often balancing detection delay and false alarm rate. The OCD problem in dynamic systems arises in various fields, including detecting false data injection attacks (FDIAs) in smart grids and identifying anomalies in wireless networks, as well as other information technology (IT) systems \cite{ yang2016deep, dou2019pc}.

Compared to OCD problems under stationarity assumptions, OCD in dynamic systems is particularly challenging because the inherent dynamics of the system can induce significant distributional changes over time, even in the absence of external events. These time-varying characteristics of pre- and post-change distributions complicate the task of reliably identifying abrupt changes.
Existing studies \cite{zhangjiangfan, adaptive_cusum, methodin14} often presuppose precise knowledge of the system. However, in practical scenarios, the parameters of the system commonly exhibit change-unrelated uncertainties due to both imperfect information acquisition and inherent dynamics.
For example, in smart grids, environmental fluctuations can alter line admittances, leading to inaccuracies in the system matrix and reducing the effectiveness of false data injection attack detection \cite{methodin14}.
Similarly, in MIMO systems, aging, quantization, and estimation errors often hinder accurate channel matrix estimation \cite{error_1}.
These challenges highlight the need for a robust detection framework with solid theoretical support.

The performance of OCD methods depends not only on the delay introduced by the detection procedure itself, but also on the computational efficiency of the algorithm, as its execution time may span a significant number of data arrivals, especially when the system dimension scales up, thereby affecting how promptly a change can be declared.
Unlike offline methods \cite{yu2024maskcd, li2024new}, which benefit from pre-collected data and parallel computation, OCD must process streaming data under real-time constraints, making efficiency improvements significantly more challenging and heavily dependent on algorithmic innovations.

To improve OCD efficiency, recent works have explored approaches such as buffering a fixed number of $m$ steps of observations \cite{zhang2024data} to facilitate efficient change detection in hidden Markov models, and bandit-based selective sensing \cite{gopalan2021bandit} to reduce observation costs by querying only a subset of informative sensors at each step. However, these methods either compromise real-time responsiveness or involve partial observability, and do not address model uncertainty. While \cite{zhangjiangfan} considers unknown, time-varying changes and achieves linear complexity in the number of observations, but its scalability with respect to system dimension remains unexamined.

Overall, the limited research on OCD methods under system uncertainty, combined with the significant challenge of improving detection efficiency in strict online settings, motivated us to design a robust and computationally efficient OCD algorithm that also scales effectively with system dimensionality. The key contributions of this paper are summarized as follows:

\begin{itemize}
    \item We propose RoS-Guard, a robust OCD algorithm that accounts for system uncertainty with theoretical guarantees.
    
    \item To improve detection efficiency and ensure scalability in large-scale systems, a method leveraging neural unrolling and GPU acceleration is introduced.
    
    \item Theoretical analysis of RoS-Guard’s performance is presented, and extensive experiments demonstrate the effectiveness and acceleration of the proposed method in large-scale settings.
\end{itemize}

\section{Related Work}

Early approaches to online change detection (OCD) often assume complete system knowledge and static pre-/post-change distributions. Classical methods such as those in \cite{adaptive_cusum, methodin14} adopt CUSUM-type statistics under known dynamics, with \cite{methodin14} further assuming fixed change locations—limiting applicability in dynamic settings. To handle model uncertainty, \cite{Minimax} studies robust OCD using least favorable distributions (LFDs) within known uncertainty sets, while \cite{Misspecified} relaxes the assumptions, offering theoretical guarantees for misspecified detection rules. Other works \cite{QCD_in_cyber_attack, zhangjiangfan} rely on residual-based detection, but often face challenges such as ill-conditioned covariances or relaxed formulations that hurt detection quality.

More recently, \cite{ULR} introduced the Uncertain Likelihood Ratio (ULR) test to handle partially known distributions, followed by \cite{WULR}, a faster variant. \cite{wasserstein} adopts a non-parametric Wasserstein ambiguity set to improve robustness. While these methods advance uncertainty handling, they often make simplifying assumptions: known ambiguity sets, independent or Markovian dynamics, fixed detection structure, or low-dimensional observations.

Emerging studies also consider computational efficiency of OCD algorithms. For example, \cite{gopalan2021bandit} develops bandit-style methods with partial observations, and \cite{zhang2024data} accelerates detection in hidden Markov models via buffered schemes. Yet these methods, they do not consider uncertainties inherent in the system parameters, which is critical for robust detection in practical dynamic environments. Unlike offline methods that exploit data parallelism, online change detection relies heavily on fast optimization to ensure real-time performance. Neural unrolling has been shown effective in accelerating large-scale constrained optimization with GPU support \cite{shi2021algorithm, he2022unsupervised, chen2025gpu}. 

Based on the existing research, we consider online change detection under more general system uncertainty conditions. Inspired by neural unrolling, we focus on designing scalable algorithms that efficiently handle large-scale and complex system settings.

\section{Problem Statement}

We consider the online change detection (OCD) problem in uncertain dynamic systems. The goal is to detect distributional changes in a sequence of observations as quickly as possible while satisfying a false alarm constraint. Our modeling follows Lorden’s minimax optimality criterion for worst-case detection delay.

Let $\Gamma$ denote the stopping time when a change is declared. The detection performance is measured by Lorden’s worst-case Average Detection Delay (ADD) \cite{Lordenminimaxproblem}:
\begin{equation}
J(\Gamma) \triangleq \sup_{\tau} \underset{\mathcal{F}{\tau}}{\operatorname{ess} \sup } ~ \mathbb{E}{\tau}\left[(\Gamma-\tau)^{+} \mid \mathcal{F}{\tau}\right],
\label{eq:delay}
\end{equation}
where $\mathcal{F}{\tau}$ is the $\sigma$-algebra generated by observations up to time $\tau$, and $\mathbb{E}{\tau}$ is the conditional expectation given a change occurs at time $\tau$. To control false alarms, the expected run length under no change, $\mathbb{E}{\infty}[\Gamma]$, must exceed a threshold $\beta > 0$. Thus, the OCD problem is formulated as:
\begin{equation}
\inf_{\Gamma} J(\Gamma) \quad \text{subject to} \quad \mathbb{E}_{\infty}[\Gamma] \geq \beta.
\label{eq:ocd_form}
\end{equation}

We focus on systems where the observation at time $t$ follows a dynamic linear model with time-varying parameters:
\begin{equation}
\mathbf{x}^{(t)} =
\begin{cases}
\mathbf{H} \boldsymbol{\theta}^{(t)} + \mathbf{n}^{(t)}, & t < t_a, \\
\left(\mathbf{H} + \Delta \mathbf{H}^{(t)}\right)\boldsymbol{\theta}^{(t)} + \mathbf{n}^{(t)}, & t \geq t_a,
\end{cases}
\label{eq:model}
\end{equation}
where $\mathbf{x}^{(t)} \in \mathbb{R}^M$ is the observed signal, $\boldsymbol{\theta}^{(t)} \in \mathbb{R}^N$ is an unknown system state, and $\mathbf{H} \in \mathbb{R}^{M \times N}$ is the nominal system matrix. The noise $\mathbf{n}^{(t)}$ is i.i.d. Gaussian, i.e., $\mathbf{n}^{(t)} \sim \mathcal{N}(\mathbf{0}, \sigma^2 \mathbf{I}_M)$. The post-change system is affected by unknown perturbations $\Delta \mathbf{H}^{(t)}$ from time $t_a$ onward.

Unlike previous works, we recognize that the precise knowledge of system matrix is typically unattainable in practical scenarios, and therefore assume it belongs to an uncertainty set $\mathcal{S}$, i.e., $\mathbf{H} \in \mathcal{S}$.
Let change in observations as $\mathbf{a}^{(t)} \triangleq \Delta\mathbf{H}^{(t)}\boldsymbol{\theta}^{(t)}$. The goal of the OCD is to detect the injected vector $\mathbf{a}^{(t)}$ as soon as possible after it actually occurs at time instant $t_{a}$.

Model (\ref{eq:model}) can be naturally generalized to long-range temporal systems by allowing the hidden state $\boldsymbol{\theta}^{(t)}$ to encode information from multiple past time steps, and by scaling up the dimensions of both $\mathbf{H}$ and $\boldsymbol{\theta}^{(t)}$ to capture richer dependencies and longer temporal horizons.
It also captures a broad class of uncertain linear systems where both the state vector and observation model evolve over time, and the change manifests as a structured perturbation to the system matrix (e.g. wireless MIMO systems\cite{imperfect_CSI_1,imperfect_CSI_2}), smart grid systems\cite{li2015cooperative}). 
\section{RoS-Guard}

This section introduces RoS-Guard, focusing on the problem modeled in (\ref{eq:model}). Specifically, (i) we begin with the reformulation of the Generalized Log-Likelihood Ratio (GLLR), which serves as the statistical evidence for detecting changes; (ii) the estimation of GLLR with system uncertainty is then formulated as a mixed-integer quadratic programming (MIQP) problem; (iii) to address computational challenges in high-dimensional settings, a relaxation and decomposition strategy is applied; and (iv) neural unrolling is employed to develop a scalable approximate solver that leverages GPU acceleration for fast computation. Finally, the overall workflow of RoS-Guard is summarized.

\subsection{Generalized Log-likelihood Ratio}
\label{GLLR}

From the model (\ref{eq:model}), the unknown and time-varying nature of $\boldsymbol{\theta}^{(t)}$ makes the component of $\mathbf{a}^{(t)}$ within the column space of $\mathbf{H}$ intrinsically unobservable. To address this, we extract the component of $\mathbf{a}^{(t)}$ that lies in the orthogonal complement of $\mathcal{C}(\mathbf{H})$ by introducing  
\begin{equation}
\boldsymbol{\mu}^{(t)} \triangleq \mathbf{P}_{\mathbf{H}}^{\perp} \mathbf{a}^{(t)} \in \mathcal{C}^{\perp}(\mathbf{H}),
\label{3-4}
\end{equation}  
where $\mathbf{P}_{\mathbf{H}}^{\perp} \triangleq \mathbf{I} - \mathbf{H}(\mathbf{H}^T\mathbf{H})^{-1} \mathbf{H}^T$ is the projection matrix onto $\mathcal{C}^{\perp}(\mathbf{H})$. To ensure the reliability of change detection, we assume a bounded magnitude for the nonzero entries of $\boldsymbol{\mu}^{(t)}$ i.e.
$\rho_L \leq |\mu_m^{(t)}| \leq \rho_H, m \in \mathcal{U}^{(t)},$
where $\mathcal{U}^{(t)}$ denotes the support of $\boldsymbol{\mu}^{(t)}$. The lower bound $\rho_L$ filters out noise-induced small perturbations, while the upper bound $\rho_H$ guards against extreme outliers due to model mismatch or measurement errors.

By estimating $\boldsymbol{\theta}^{(t)}$ and $\Delta\mathbf{H}^{(t)}$ via maximum likelihood (MLE) \cite{GCUSUM}, the detection rule follows the generalized likelihood ratio (GLR) form:
\begin{equation}
   \Gamma_R = \min \left\{ K : \max_{1 \leq k \leq K} \Lambda_k^{(K)} \geq h \right\},
   \label{GCUSUM}
\end{equation}
where $\Lambda_k^{(K)}$ denotes the log-likelihood ratio statistic between the pre- and post-change models.

Incorporating the orthogonal component $\boldsymbol{\mu}^{(t)}$ defined in (\ref{3-4}), the post-change model can be reformulated to depend only on the detectable directions. As a result, the log-likelihood ratio becomes:
\begin{equation}
\begin{aligned}
&\Lambda_k^{(K)} \triangleq \\
&\sup_{\{\mathcal{U}^{(t)}\},\, \mathbf{H}} 
\ln \frac{
\sup\limits_{\boldsymbol{\theta}^{(t)},\, \Delta\mathbf{H}^{(t)},\, \boldsymbol{\mu}^{(t)}} 
\prod_{t=k}^{K} f_q\left(\mathbf{x}^{(t)} \mid \boldsymbol{\theta}^{(t)}, \mathbf{H}, \Delta\mathbf{H}^{(t)}\right)}
{
\sup\limits_{\boldsymbol{\theta}^{(t)}} 
\prod_{t=k}^{K} f_p\left(\mathbf{x}^{(t)} \mid \boldsymbol{\theta}^{(t)}, \mathbf{H}\right)
},
\end{aligned}
\label{test statistic}
\end{equation}
where $f_p$ and $f_q$ are Gaussian densities corresponding to the pre- and post-change models, respectively.
$\Lambda_k^{(K)}$ can be further simplified to a sum of scalar statistics, i.e.,
$\Lambda_k^{(K)} = \sum_{t=k}^{K} v_t,$ where each $v_t$ represents the instantaneous evidence for change, and is derived as follows:
\begin{equation}
\begin{split}
v_t = \sup_{\mathcal{U}^{(t)}} \sup_{\boldsymbol{\mu}^{(t)},\mathbf{H}} 
\quad & \frac{1}{2\sigma^2} \left\{ 2\left(\boldsymbol{\mu}^{(t)}\right)^T \boldsymbol{x}^{(t)} - \left\| \boldsymbol{\mu}^{(t)} \right\|_2^2 \right\} \\
\text{s.t.} \quad & \rho_L \leq \left| \mu_m^{(t)} \right| \leq \rho_U, \quad \forall m \in \mathcal{U}^{(t)}, \\
& \mathbf{H}^T \boldsymbol{\mu}^{(t)} = \mathbf{0}, \\
& \mathbf{H} \in \mathcal{S}.
\end{split}
\label{v_t}
\end{equation}
Recursively, the accumulated evidence statistic $V_K$ can be computed as
\begin{equation} 
\begin{aligned}
V_K &\triangleq \max_{1 \leq k \leq K} \Lambda_k^{(K)} 
= \max \left\{ V_{K-1}, 0 \right\} + v_K,
\end{aligned}
\label{V_K}
\end{equation}
where $V_0=0$ and a change is declared when $V_K \geq h$.

Full derivation of $\Lambda_k^{(K)}$ is provided in the extended version (see the Links section)

%%%%%%%%%%%%%%%%%%%%%%%%%%%%%%%%

\subsection{GLLR with System Uncertainty}

\label{sec:PR}
In this subsection, we focus on the uncertainty of the system matrix and detail the reformulation of equation (\ref{v_t}) as an MIQP problem. To simplify notation, we omit the time superscripts of variables.

We first reformulate the original $\sup$-based objective $v_t$ as an equivalent $\inf$ minimization problem for tractability and we represent the support set of the decision vector $\boldsymbol{\mu}$ using a binary vector $\boldsymbol{u} \in \{0,1\}^N$, where $u_m = 1$ indicates that the $m$-th component of $\boldsymbol{\mu}$ belongs to the active support set $\mathcal{U}$. Subsequently, we decomposed $\boldsymbol{\mu}$ into two nonnegative components with complementarity constraint, i.e, $\boldsymbol{\mu} = \boldsymbol{\mu}^{+} - \boldsymbol{\mu}^{-}, {\boldsymbol{\mu}^{+}}^T\boldsymbol{\mu}^{-} = \mathbf{0}$.

Then we employ the versatile constraint-wise uncertainty paradigm \cite{yang2014distributed,priceofrobustness}, which decouples the uncertainties among different rows in the system matrix $\mathbf{H}$, i.e., denoting its $i$-th column as $\mathbf{h}_i$, with each column lying in an uncertainty set $\mathcal{S}_i$. We further relax the hard constraint $\mathbf{H}^T \boldsymbol{\mu} = \mathbf{0}$ to the following robust form:
\begin{equation}
\mathbf{\bar{h}}_i^T \boldsymbol{\mu}_i + \max_{\mathbf{h}_i \in \mathcal{S}_i} (\mathbf{h}_i - \mathbf{\bar{h}}_i)^T \boldsymbol{\mu}_i \leq \varepsilon_i,
\end{equation}
where $\mathbf{\bar{h}}_i$ denotes the estimated nominal value of $\mathbf{h}_i$. 
For clarity, we illustrate our approach using the general polyhedral uncertainty set $\mathcal{S}_{i}=\left\{\boldsymbol{h}_{i} \mid \mathbf{D}_{i} \boldsymbol{h}_{i} \leq \boldsymbol{d}_{i} \right\},~i = 1, \cdots, N$. 
Our method, however, readily extends to other common uncertainty sets such as ellipsoids and D-norms. Its generality is further demonstrated in our experiments on various uncertainty sets. As a result, the reformulated problem of (\ref{v_t}) becomes:
\begin{equation}
\begin{split}
v_t = &-\inf_{\mathcal{U}} \inf_{\boldsymbol{\mu}^{(t)},\mathbf{H}} 
[-\mathcal{F}(\bm{\mu},\bm{x})]\\
\text{s.t.}~ & \left\{
\begin{array}{l}
\boldsymbol{\mu} = \boldsymbol{\mu}^{+} - \boldsymbol{\mu}^{-}, {\boldsymbol{\mu}^{+}}^T\boldsymbol{\mu}^{-} = \mathbf{0} \\
\rho_{L} \boldsymbol{u} \leq \boldsymbol{\mu}^{+}+\boldsymbol{\mu}^{-} \leq \rho_{U}\boldsymbol{u},\bm{u} \in \{0,1\}^N \\
u_m= 1 ~\text{if}~ m \in \mathcal{U},~\text{else}~ u_m = 0 \\
\bar{\mathbf{h}_i}^T \boldsymbol{\mu}_i + \max_{\mathbf{h}_{i}\in \mathcal{S}_i} (\mathbf{h}_i - \mathbf{\bar{h}}_i)^T \bm{\mu}_i \leq \varepsilon_i \\
\mathbf{D}_{i} \boldsymbol{h}_{i} \leq \boldsymbol{d}_{i},
\end{array}
\right.
\end{split}
\label{v_t_reform}
\end{equation}
where $\mathcal{F}(\bm{\mu},\bm{x})= \frac{1}{2 \sigma^{2}}  \left\{ \left\|(\boldsymbol{\mu}^{+} - \boldsymbol{\mu}^{-})\right\|_{2}^{2} -2 {(\boldsymbol{\mu}^{+} - \boldsymbol{\mu}^{-}) }^{T}{\boldsymbol{x}}\right\}$

The problem remains challenging due to the presence of a maximization over an uncertainty set, which leads to a nested bilevel structure and the non-convex orthogonality constraint ${\boldsymbol{\mu}^+}^T \boldsymbol{\mu}^- = 0$. Thus, by denoting $\bm{p}_i$ the dual variable and leveraging strong duality \cite{yang2014distributed}, the maximization constraint under the polyhedron uncertainty set can be equivalently reformulated as:
\begin{equation}
\boldsymbol{p}_{i}^{T} \boldsymbol{d}_{i} \leq \varepsilon_{i}, \mathbf{D}_{i}^{T} \boldsymbol{p}_{i} = \boldsymbol{\mu}, \boldsymbol{p}_{i} \geq 0 .
\label{ro_l}
\end{equation}
Introducing a binary auxiliary variable $\boldsymbol{b} \in \{0,1\}^N$ and using the upper bound $\rho_U$ as a sufficiently large constant.  The mutual exclusiveness between $\boldsymbol{\mu}^{+}$ and $\boldsymbol{\mu}^{-}$ at each index is replaced by the following two linear inequalities: 
\begin{equation}
\boldsymbol{\mu}^{+} + \rho_{U} \boldsymbol{b} \leq \rho_{U}\boldsymbol{1}, 
\boldsymbol{\mu}^{-} - \rho_{U} \boldsymbol{b} \leq \boldsymbol{0}.
\end{equation}
Thus, \eqref{v_t} is reformulated as the following MIQP problem:

\begin{equation}
\label{MIQP}
\begin{split}
&\operatorname{min} \frac{1}{2 \sigma^{2}}  \left\{ \left\|(\boldsymbol{\mu}^{+} - \boldsymbol{\mu}^{-})\right\|_{2}^{2} -2 {(\boldsymbol{\mu}^{+} - \boldsymbol{\mu}^{-}) }^{T}{\boldsymbol{x}}
\right\} \\
&s.t.\quad  \left\{\begin{array}{lc}
\boldsymbol{p}_{i}^{T} \boldsymbol{d}_{i} \leq \varepsilon_{i},  \mathbf{D}_{i}^{T} \boldsymbol{p}_{i} = \boldsymbol{\mu}^{+} - \boldsymbol{\mu}^{-}, \forall i  \\
\rho_{L} \boldsymbol{u} \leq{   \boldsymbol{\mu}^{+}+\boldsymbol{\mu}^{-}}     \leq \rho_{U}\boldsymbol{u} \\
\boldsymbol{\mu}^{+} + \rho_{U} \boldsymbol{b} \leq \rho_{U}\boldsymbol{1}\\
\boldsymbol{\mu}^{-} - \rho_{U} \boldsymbol{b} \leq \boldsymbol{0}\\
\end{array}\right.\\
&\text{var}:  \boldsymbol{u} \in \{0,1\}^{M},\boldsymbol{b} \in \{0,1\}^{M},\boldsymbol{\mu}^{+}, \boldsymbol{\mu}^{-}, \boldsymbol{p}_{i}\geq \boldsymbol{0}.
\end{split}
\end{equation}

\subsection{Relaxation for Efficient Optimization}

Mixed-integer problems are often solved using discrete methods like branch-and-bound, whose complexity grows rapidly with problem dimension. To improve efficiency, we reformulate the problem \eqref{MIQP} via continuous relaxation. The tightness of the relaxation directly affects the approximation quality. While semidefinite programming (SDP) is a classical relaxation with tight bounds in combinatorial tasks such as max-cut \cite{SDp_cut} and graph coloring \cite{SDP_color}. Here, we propose a relaxation for problem~\eqref{MIQP},that can be theoretically proven to provide better bound than the traditional SDP relaxation. 

Since the objective function of problem \eqref{MIQP} is separable, here we first consider the one-dimensional case. By continuously relaxing $b_m$ to $[0,1]$, we arrive at the following problem formulation.
\begin{equation}
\begin{aligned}
&g(\mu_m^+,\mu_m^-,u_m,b_m,\{\boldsymbol{p}_{i}\})= \\
&\begin{cases}0,
&\text{if } u_m = \mu_m^+ = \mu_m^-=0,\\
f(\mu_m^+,\mu_m^-),
&\text{if } u_m=1,\rho_L \leq \mu_m^{+} +\mu_m^{-}\leq \rho_H ,    \\+\infty,
&\text{otherwise,}\end{cases}  \\
&(\mu_m^{+}, \mu_m^{-}, b_m,\{\boldsymbol{p}_{i}\}) \in \mathcal P,
\end{aligned}
\end{equation}
where $f$ represents the objective function of \eqref{MIQP}, with feasible region $\mathcal{P}$ defined by all constraints excluding $\rho_L u_m \leq \mu_m^+ + \mu_m^- \leq \rho_H u_m$.

Then, we compute the convex envelope $\overline{co}(g)$ by constructing the convex hull in the epigraphical space of $g$~\cite{PCF_1, PCF_2}, yielding:
\begin{equation}
\begin{aligned}
&\overline{co}(g)(\mu_m^+,\mu_m^-,u_m,b_m,\{\boldsymbol{p}_{i}\}) \\
=&\begin{cases}0,
&\text{if } u_m = \mu_m^+ = \mu_m^-=0\\
h(\mu_m^+,\mu_m^-,u_m),
&\text{if } u_m\in(0,1],\rho_L \leq \mu_m^{+} +\mu_m^{-}\leq \rho_H ,    \\+\infty,
&\text{otherwise.}\end{cases}  \\
&(\mu_m^{+}, \mu_m^{-}, b_m,\{\boldsymbol{p}_{i}\}) \in \mathcal P
\end{aligned}
\end{equation}
where $h(\cdot)=\frac{1}{2\sigma^2}\{ u_m^{-1} (\mu_m^+ - \mu_m^-)^2 - 2x_m(\mu_m^+ - \mu_m^-) \}$ is a continuous relaxation of the objection in \eqref{MIQP} and defining $0/0:=0$. By introducing auxiliary variables $\phi_m$ such that $ u_m^{-1} (\mu_m^+ - \mu_m^-)^2\leq \phi_m$, which admits a second-order cone (SOC) formulation:
\begin{equation}
\begin{split}
&\operatorname{min}  \frac{1}{2 \sigma^{2}}  \sum_{m}   \bigg\{ \phi_m - 2 {({\mu}_m^{+} - {\mu}_m^{-}) }^{T}{{x}_m}    \bigg\}
  \\
&s.t.\quad  \left\{\begin{array}{lc}
\begin{Vmatrix} {\mu}_m^{+} - {\mu}_m^{-} \\\frac{\phi_m-u_m}{2}\end{Vmatrix} \leq           \frac{\phi_m+u_m}{2},\forall m\\
\boldsymbol{p}_{i}^{T} \boldsymbol{d}_{i} \leq \varepsilon_{i},  \mathbf{D}_{i}^{T} \boldsymbol{p}_{i} = \boldsymbol{\mu}  ^{+} - \boldsymbol{\mu}  ^{-}, \forall i  \\
\rho_{L} \boldsymbol{u}   \leq{   \boldsymbol{\mu}  ^{+}+\boldsymbol{\mu}  ^{-}}     \leq \rho_{U}\boldsymbol{u}   \\
\boldsymbol{\mu}  ^{+} + \rho_{U} \boldsymbol{b}   \leq \rho_{U}\boldsymbol{1}\\
\boldsymbol{\mu}  ^{-} - \rho_{U} \boldsymbol{b}   \leq \boldsymbol{0}\\
\end{array}\right.\\
&\text{var}:  \boldsymbol{u}\in[0,1]^N ,\boldsymbol{b}\in[0,1]^N ,\bm{\phi},\boldsymbol{\mu}  ^{+}, \boldsymbol{\mu}  ^{-}, \boldsymbol{p}_{i}\geq \boldsymbol{0}.
\end{split}
\label{PC_formulation}
\end{equation}
Problem \eqref{PC_formulation} is a Second-Order Cone Programming (SOCP) problem. We can demonstrate that this relaxation offers superior bound compared to traditional SDP relaxation.This is due to the absence of integer variables in the objective function, which leads to the degeneration of the SDP relaxation into a quadratic programming problem. Full derivations and equivalence proofs are deferred to the extended version (see the Links section).

\subsection{Neural Unrolling for GPU-acceleration}
Neural unrolling unfolds iterative optimization algorithms into trainable neural network layers, enabling efficient and adaptive solution approximation\cite{shi2021algorithm, he2022unsupervised, chen2025gpu}. This approach naturally supports GPU acceleration, allowing parallel computation that significantly improves scalability and speed. We first construct the Lagrangian function of problem \eqref{PC_formulation}, as follows:
\begin{equation}
\label{eq:lagrangian}
\begin{aligned}
&\min \mathcal{L}(\boldsymbol{\phi}, \boldsymbol{u}, \boldsymbol{b}, \boldsymbol{\mu}^+, \boldsymbol{\mu}^-, \boldsymbol{p}, \boldsymbol{\lambda}) \\
= &\min \frac{1}{2\sigma^2}  
\sum_m \left\{ \phi_m - 2 (\mu_m^+ - \mu_m^-)^\top x_m 
\right\} \\
&+ \sum_{\ell} \left\{ \sum_j \lambda_j g_j(\cdot) \right\}_\ell,\\
&\text{var}: \bm{\phi}, \boldsymbol{u} ,\boldsymbol{b} ,\boldsymbol{\mu}  ^{+}, \boldsymbol{\mu}  ^{-}, \boldsymbol{p},\bm{\lambda}.
\end{aligned}
\end{equation}
where $g_j(\cdot)$ denotes the constraint functions and $\lambda_j>0$ are the Lagrange multipliers, which acts as penalty coefficients.

The optimization of problem \eqref{eq:lagrangian} proceeds by alternately updating the primal variables and the dual variables $\lambda$, which represent the penalty coefficients for constraint violations. By defining an RNN network, the optimization process can be represented by the parameter updates of the RNN network with $K$ layers, as illustrated in \cref{fig:neural_unfolding}. 
\begin{figure}[h]
    \centering
    \includegraphics[width=\columnwidth,trim=90 50 90 10, clip]{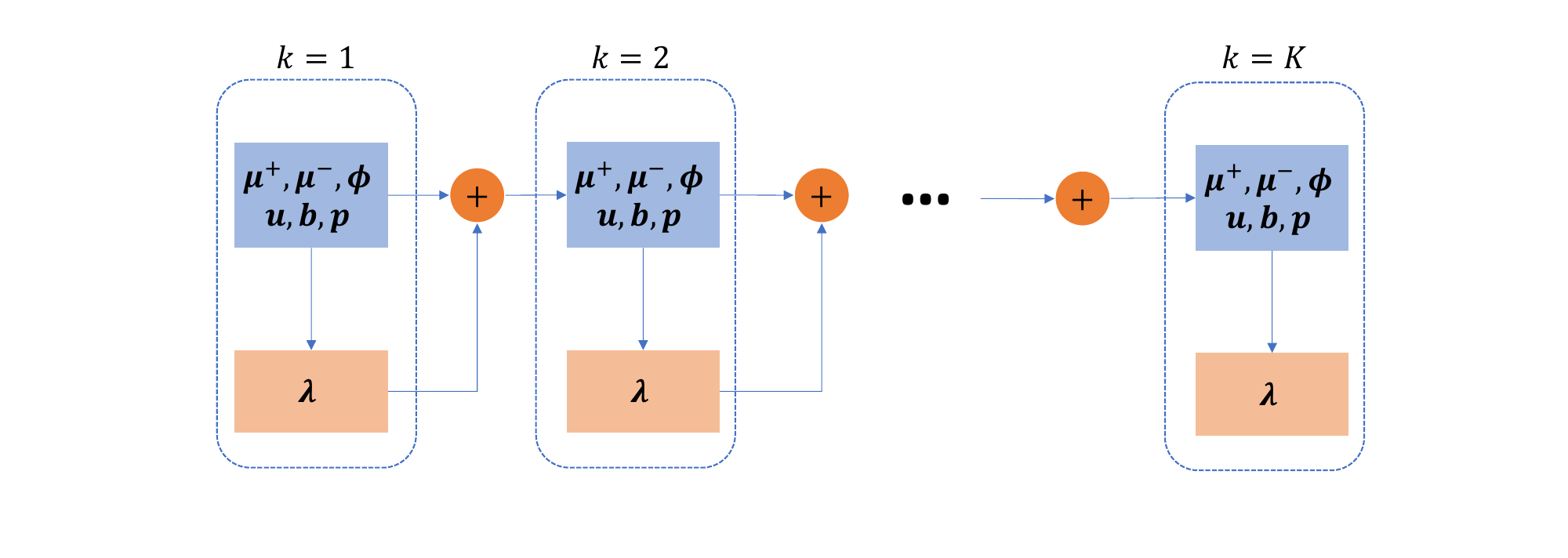} % 图片路径和大小
    \caption{Illustration of the Unrolled Optimization Network}
    \label{fig:neural_unfolding}
\end{figure}
The learnable parameters within each layer capture critical algorithmic components (e.g. step sizes). The optimization loss is defined based on objective of \eqref{eq:lagrangian}, guiding the training process. The iterative updates proceed until the difference between consecutive iterations falls below a threshold $\epsilon_\Delta$. Overall, the whole process of RoS-Guard is summarized in \cref{RoS-Guard}. 

\begin{algorithm}[H]
\caption{RoS-Guard algorithm}
\label{RoS-Guard}
\KwIn{$\{\boldsymbol{x}_t\}$, $\tau$, $\mathbf H, \rho_L,\rho_U,\sigma^2, \{ \varepsilon_{i} \}, \{\mathcal{S}_{i}\} $, $h,\epsilon$}
Initialize $t \leftarrow 0$, $V_0 \leftarrow 0$.\\
\If{System is small-scale}{ 
\Repeat{$V_t \geq h $}{
    $t \leftarrow t+1.$

  % Approximate the value of $v_t$ by solving problem (\ref{PC_formulation}).
  Based on the relaxation (\ref{PC_formulation}), employ branch and bound method to solve problem (\ref{MIQP}), and thereby obtain the value of $v_t$.

  Update  $V_t \leftarrow max\{V_{t-1},0\} + v_t.$
}
  \KwOut {$T_G \leftarrow t$, declare the change point $T_G$. }
}

\ElseIf{System is large-scale}{
    \Repeat{$V_t\geq h$}{
    Initialize neural model parameters\;\\
    \Repeat{$|\text{Loss}_{\text{prev}} - \text{Loss}_{\text{curr}}| < \epsilon$}{
        optimize \eqref{eq:lagrangian} with unrolling network\;
    }
    Obtain $v_t^{(\ell)}$ with optimized parameters and aggregate $v_t = \sum_{\ell=1}^{L} v_t^{(\ell)}$\;
    
    Update the decision statistic :  $V_t$
    
    }

    \KwOut {$T_G \leftarrow t$, declare the change point $T_G$. }
}

\end{algorithm}

\section{Theoretical Analysis}
As shown in Problem Statement \eqref{eq:delay} \eqref{eq:ocd_form}, OCD aims to minimize the worst-case ADD while satisfying a lower-bound constraint on the expected false alarm period (FAP). In this section, we provide a theoretical analysis of our proposed algorithm by establishing performance guarantees. Specifically, we first derive a sufficient condition to ensure the expected FAP constraint is met, and then present an upper bound on the worst-case ADD under any given detection threshold.

\subsection{Condition to Meet Expected False Alarm Period
Constraint}
According to (\ref{GCUSUM}), the change is declared at the first time $V_K \geq h$. Therefore, the value of the threshold $h$ holds great significance in the detection process.
In this subsection, we derive a sufficient condition for our method. For any prescribed lower bound $\gamma$ on the FAP, this condition provides a guideline for the selection of $h$ to ensure that the FAP constraint can be satisfied.   
\begin{theorem}
\label{theorem_2}
Suppose $\mathbf{x}^{(t)}$ is upper bounded, i.e, $ \| \mathbf{x}^{(t)}\|_2^2\leq \alpha$. The expected false alarm period of RoS-Guard is larger than $\gamma$ if :
\begin{equation}
    h \geq \frac{\alpha}{2\sigma^2}   \gamma .
    \label{eq_theorem2}
\end{equation}
\end{theorem}

\begin{proof}[Proof Sketch]
Recall that $\Gamma_R$ represents the stopping time of RoS-Guard as shown in (\ref{GCUSUM}).
The core of proof resides in establishing the relationship between $\mathbb{E}_{\infty}\{V_{\Gamma_R}\}$ and $\mathbb{E}_{\infty}\{  \Gamma_R \}$, which is achieved by the upper bound on $v_t$.
When the false alarm is declared, we have $\mathbb{E}_{\infty}\{V_{\Gamma_R}\} \geq h$, which thereby elucidates the relationship between the $\mathbb{E}_{\infty}\{  \Gamma_R \}$ and $h$.
The complete proof can be found in \textbf{Proof 2}.
\end{proof}

\subsection{Upper Bound on the Worst-Case Expected Detection
Delay}
For OCD task, it's imperative to evaluate the worst-case expected detection delay of the detector.
In this subsection, we provide the theoretical analysis of the worst-case expected detection delay of the proposed method.
Define $J(\Gamma_R)$ as Lorden's worst-case expected detection delay.
The subsequent theorem concerning the worst-case expected detection delay can be obtained.
\begin{theorem}
\label{theorem_3}
Let $d_i$ represents the diameter of uncertainty set $\mathcal{U}_{i}$.
For any threhold $h$, by employing Wald’s approximations \cite{tartakovsky2014sequential}, set $\varepsilon_{i} \geq d_i \rho_H M^{\frac12}$, the worst-case expected detection delay of SDPCUSUM and BBCUSUM can be bounded as follows,
\begin{equation}
J(\Gamma_{R}) \leq \frac{2 h {\sigma}^2}{ {\rho_L}^2}, 
\label{eq_theorem3}
\end{equation}
\end{theorem}
\begin{proof}[Proof Sketch]
We begin by introducing a lower bound on the expectation of $v_t$ when change occurs at a time instant $t_a$ based on the bounds of $\boldsymbol{\mu}$. Subsequently, we proceed to derive an upper bound on $\mathbb{E}_{t_{a}}\{(\Gamma_R-t_{a}+1)^{+}  \mid \mathcal{F}_{t_{a}-1} \}$ for the proposed method. Next, we demonstrate that the proposed method achieves the equalizer rule given the pre- and post-change model elaborated in \eqref{eq:model}, Finally, leveraging these insights and employing Wald’s approximations, we substantiate Theorem 2. The complete proof can be found in the extended version (see the Links section).
\end{proof}

\section{Experiments}
\label{experiments}
To evaluate the effectiveness of our proposed RoS-Guard algorithm, we conduct experiments on two representative online change detection scenarios: (i) attack injection detection in smart grids, and (ii) channel blockage detection in MIMO wireless systems. We compare our method against two state-of-the-art baselines, RGCUSUM and CyberQCD, which are compatible with our system modeling. Furthermore, to assess the scalability of RoS-Guard in different scale systems, we measure and compare the detection latency under varying observation dimensions.
\subsection{Datasets}

\textbf{Dataset \uppercase\expandafter{\romannumeral1}: } 
We follow \cite{adaptive_cusum} to formulate FDIA detection in smart grids using the dynamic DC power flow model in \eqref{eq:model}. The system includes $N+1$ buses and $M$ meters, with state $\boldsymbol{\theta}^{(t)} \in \mathbb{R}^N$ and observations $\mathbf{x}^{(t)} \in \mathbb{R}^M$. The measurement matrix $\mathbf{H} \in \mathbb{R}^{M \times N}$ is determined by grid topology and line susceptance \cite{kosut2011malicious, cui2012coordinated}. We conduct experiments using the IEEE-14 bus system, a standard benchmark. The system state is initialized using "case14" in MATPOWER \cite{matpower}, and we simulate attacks by injecting random vectors into $\mathbf{x}^{(t)}$ from time $t_a$. Following \cite{methodin14}, we assume the measurement matrix is fully known except in one sub-region. Uncertainty sets and perturbed $\mathbf{H}$ are detailed in the extended version (see the Links section).\\
\textbf{Dataset \uppercase\expandafter{\romannumeral2}: } 
MIMO technology leverages multiple antennas at both the transmitter and receiver to enhance transmission efficiency via multiple signal paths \cite{MIMO1}. However, these paths are susceptible to blockages from static or dynamic obstacles such as buildings, vehicles, or humans \cite{MIMO2}. Let $\mathbf{x}^{(t)} \in \mathbb{R}^{M}$ be the received signal vector and $\boldsymbol{\theta}^{(t)} \in \mathbb{R}^{N}$ the transmitted signal vector, where $M$ and $N$ are the numbers of receive and transmit antennas, respectively. The channel matrix $\mathbf{H} \in \mathbb{R}^{M \times N}$ contains the channel gains between antenna pairs. Due to environmental uncertainties, these coefficients are typically imprecise. We adopt a $2 \times 4$ MIMO testbed based on USRP devices. As shown in the extended version (see the Links section), various blockage scenarios are considered. For each case, we continuously collect the received signals and apply OCD detectors to identify blockage events.

\subsection{Performance Evaluation}
In the introduction, we underscored the limitations of conventional OCD methods due to their reliance on certain assumptions. These assumptions encompass constant distributions before and after a change, as well as specific conditions for the system state. Consequently, these methods prove unsuitable for addressing the problem under investigation.  
We implemented the Adaptive CUSUM algorithm \cite{adaptive_cusum} In our experiments to exemplify this fact. 
Furthermore, a performance comparison is conducted between our algorithm, the method proposed in \cite{methodin14} (referred to as CyberQCD for simplicity), and the RGCUSUM algorithm \cite{zhangjiangfan}, which claims to be the state-of-the-art detector for the OCD in dynamic systems.
To better showcase the effectiveness and superiority of the proposed method, we approximating the value of $v_t$ by directly using the optimal objective function value of (\ref{PC_formulation}).

After executing each detector 500 runs, the experimental results on data I and data II are visually presented in \cref{fig:attack_detection_and_uncertainties} and \cref{fig:blockage_detection}.

\begin{figure}[h]
    \centering
    \includegraphics[width=\columnwidth]{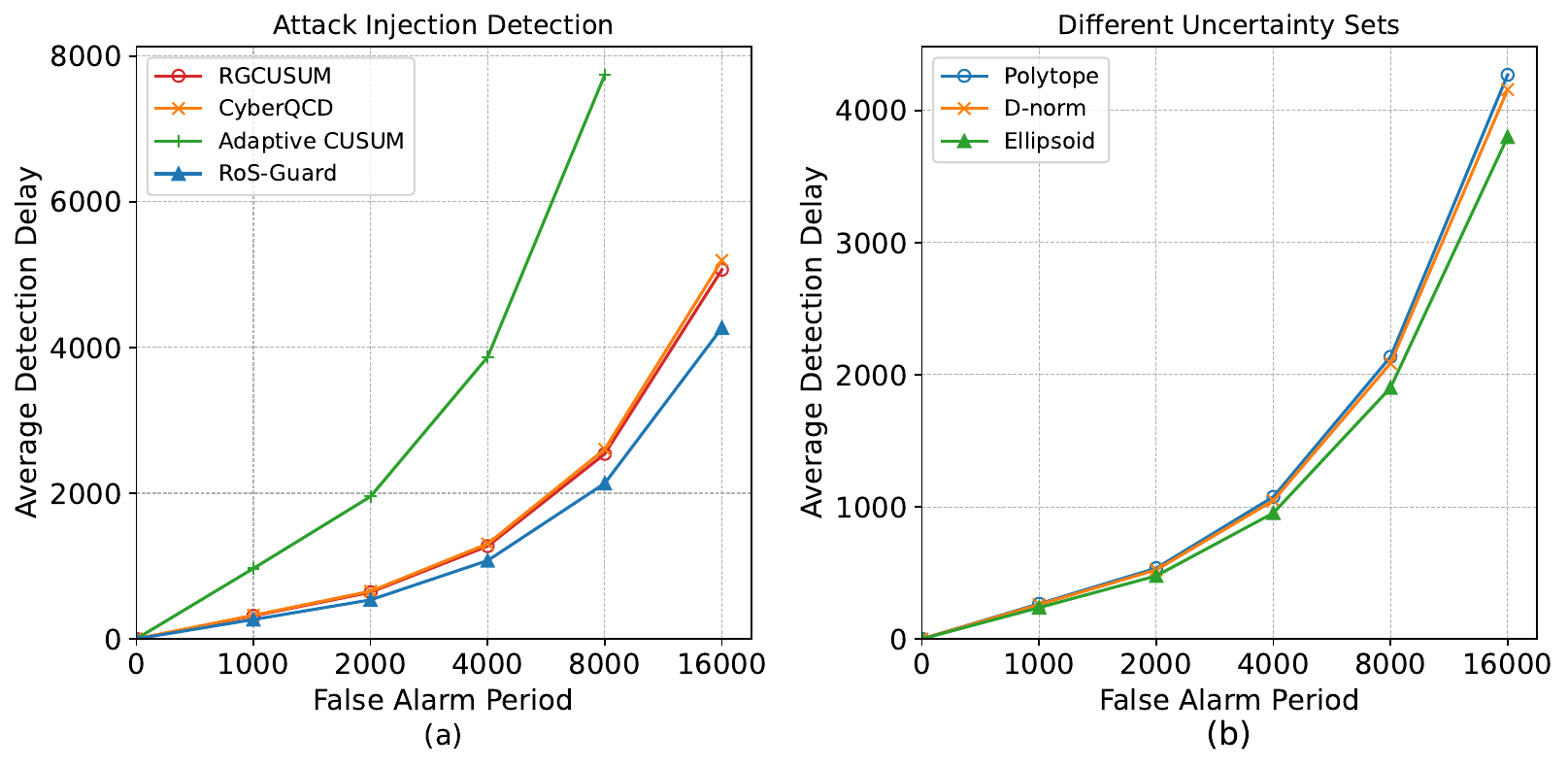} % 图片路径和大小
    \caption{(a) Performance Comparison with Attack Injection Detection in Smart Grids System, (b) Performance Comparison with Various Uncertainty Sets. }
    \label{fig:attack_detection_and_uncertainties}
\end{figure}

\begin{figure}[h]
    \centering
    \includegraphics[width=\columnwidth]{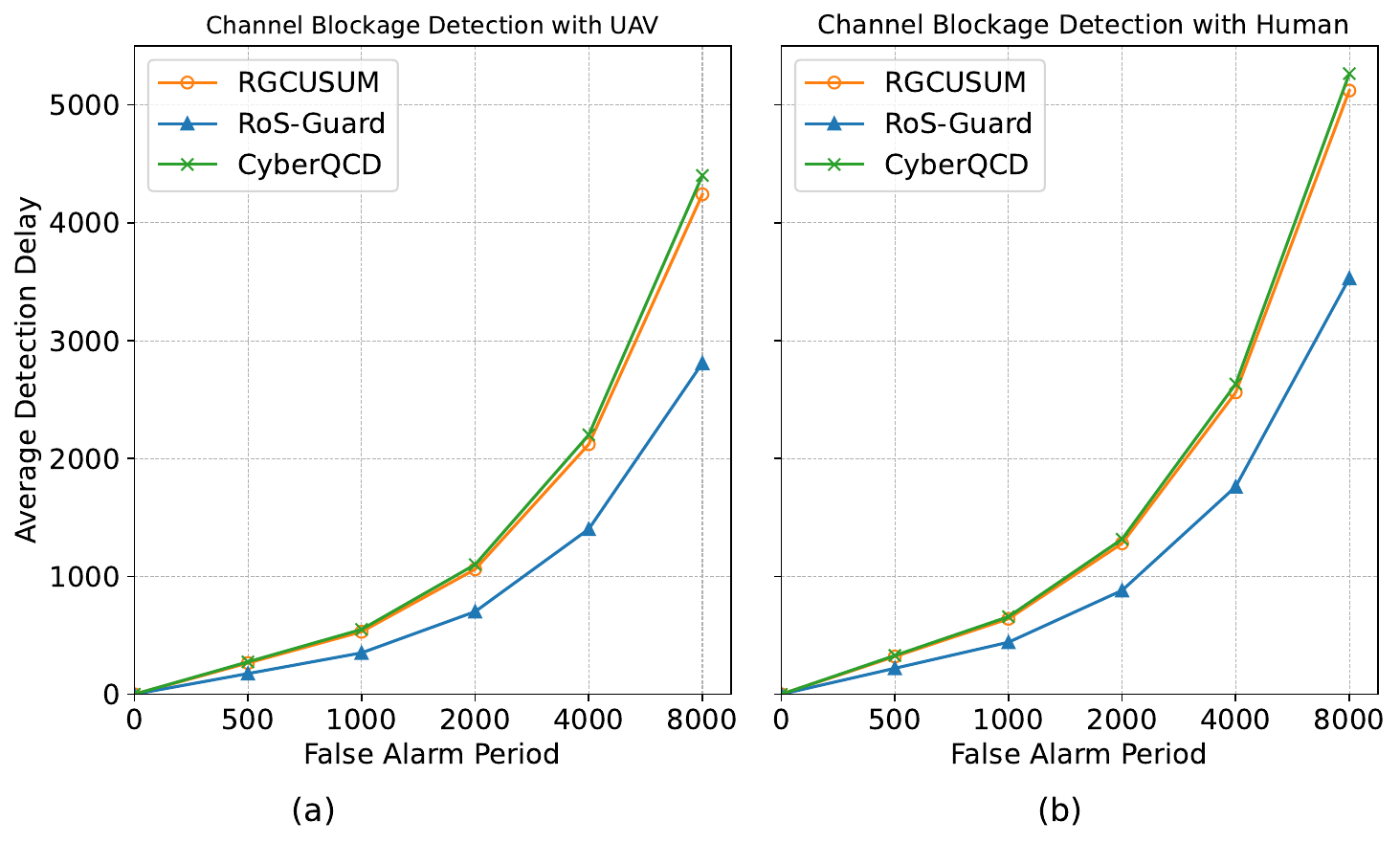} % 图片路径和大小
    \caption{Performance Comparison with Blockage Detection in Wireless MIMO System.}
    \label{fig:blockage_detection}
\end{figure}

Following conventional metrics for online change detection, we evaluate the ADD of different detectors under the same FAP. 
Recall that FAP signifies the point at which the detector stops when no change is detected, thus serving as a measure to assess the risk of false alarms. It is evident from \cref{fig:attack_detection_and_uncertainties} (a) and \cref{fig:blockage_detection} that our method consistently exhibits a smaller average detection delay for any given FAP, which emphasizes the superior performance achieved by RoS-Guard.

From \cref{fig:attack_detection_and_uncertainties}(a), we can observe that the performance of the adaptive CUCUM algorithm is unacceptable.  This can primarily be attributed to the inconsistency between its assumptions and the experiment setup. Specifically, the adaptive CUCUM algorithm assumes a Gaussian distribution for the systems state, which does not hold in our experiments. 
RGUCUSM and CyberQCD show comparable performance, with both falling short in comparison to the performance of the RoS-Guard. 
This result is in accordance with our expectation, as both the RGUCUSM and CyberQCD methods are developed based on the perfectly known $\mathbf{H}$.
When the system matrix $\mathbf{H}$ is inaccurate, their detection performance rapidly deteriorates. 
In contrast, the RoS-Guard algorithm effectively handles uncertainties within the system matrix, ensuring robust and reliable detection performance.
Furthermore, we conducted additional experiments to showcase the performance of the RoS-Guard algorithm when the system matrix $\mathbf{H}$ is assumed to belong to different uncertainty sets. These uncertainty sets, which are commonly employed in practical applications, include ellipsoid uncertainty sets, D-norm uncertainty sets, and polyhedron uncertainty sets \cite{yang2014distributed}. Detailed information regarding the settings of these uncertainty sets can be found in the extended version (see the Links section). The experiments were conducted with a total of 500 Monte Carlo runs. As depicted in \cref{fig:attack_detection_and_uncertainties} (b), the RoS-Guard algorithm consistently demonstrates superior performance across diverse uncertainty sets, underscoring its remarkable ability to generalize and adapt to varying conditions.

\subsection{Evaluation of GPU Parallel Computing}

To compare the detection time differences between CPU- and GPU-based algorithms under varying system scales, we follow the setup of Dataset I and consider attack injection under system observations of size $2^k$
. By tuning parameters to ensure comparable detection performance, we evaluate and compare the execution times of the CPU and GPU algorithms. The average runtime is computed over 50 randomized trials. The detailed experimental settings are provided in the extended version (see the Links section).

Experimental result in \cref{fig:detection_time} shows that our GPU-based neural unrolling algorithm exhibits a significant speed advantage when the system scale is large. Specifically, when the system reaches a scale of $2^8$, the GPU-based parallel algorithm achieves more than a 20× speedup.

\begin{figure}[h]
    \centering
    \includegraphics[width=0.6\columnwidth]{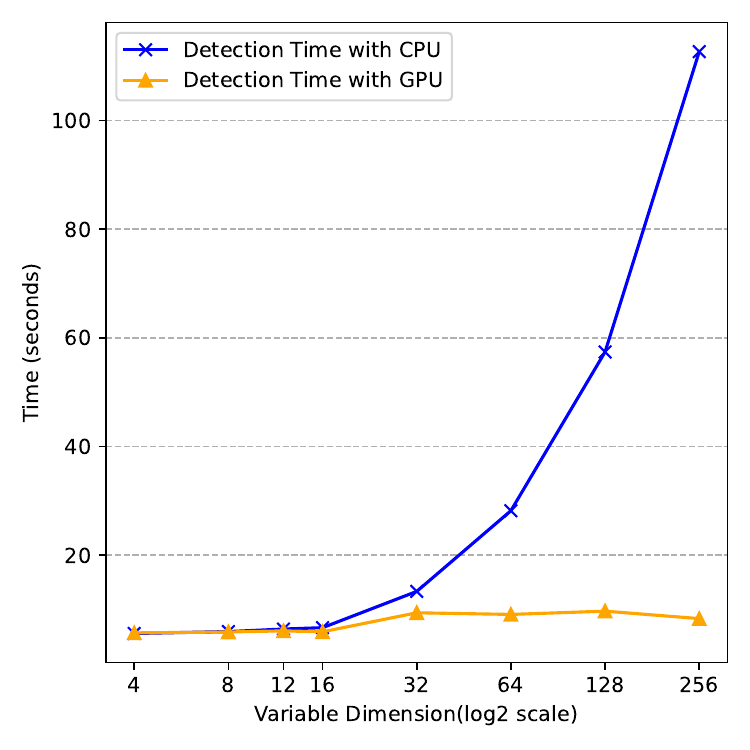} % 图片路径和大小
    \caption{Detection Time Comparison between Neural Unrolling with GPU and MIQP with CPU in Different System Scale.}
    \label{fig:detection_time}
\end{figure}
It is worth noting that if we slightly violate the online change detection setting by grouping a small number 
$T$ of consecutive observations into a batch, the GPU-based method does not incur additional execution time. This indicates that when the algorithm's execution time is comparable to the system's observation interval, the efficiency advantage of the GPU-based method scales approximately by a factor of $T$

\subsection{Numerical Verification of \cref{theorem_2} and \cref{theorem_3}}
In this subsection, we numerically verify \cref{theorem_2} and \cref{theorem_3} on Data \uppercase\expandafter{\romannumeral1}. The number of Monte Carlo runs is 100. 
In order to verify \cref{theorem_2}, for any prescribed $\gamma$, we set the value of $h$ to be the lower bound calculated from the right side of inequality in (\ref{eq_theorem2}). And then we evaluate the actual average FAP of RoS-Guard with the same $h$.  The results are shown in \cref{fig:numerical_verification}(a).
It is seen that under every setting of $(\rho_{L},\rho_{H})$, the actual FAP is always larger than $\gamma$ for our detector, which confirms \cref{theorem_2}.
% To further demonstrate the superiority of our lower bound, we also compare it with the lower bound provided in \cite{zhangjiangfan}.
% The numerical results under different $(\rho_{L},\rho_{H},\sigma )$ can be found in \cref{cp_lower_bound}. 
% It is evident that our lower bound is significantly smaller than that provided in \cite{zhangjiangfan}. This implies that as long as the threshold $h$ exceeds our lower bound, the FAP constraint of the detectors can be guaranteed.

To demonstrate the validity of \cref{theorem_3}, we compute the upper bounds on $J(\Gamma_R)$ for different $h$ values by employing (\ref{eq_theorem3}).  And then for each value of $h$, we numerically calculate the corresponding actual ADD of the proposed method.  The experimental results are shown in \cref{fig:numerical_verification}(b). It is seen that the actual ADD of RoS-Guard is consistently smaller than the upper bounds on $J(\Gamma_R)$, which validates \cref{theorem_3}.

\begin{figure}[htbp]
    \centering
    \includegraphics[width=\columnwidth]{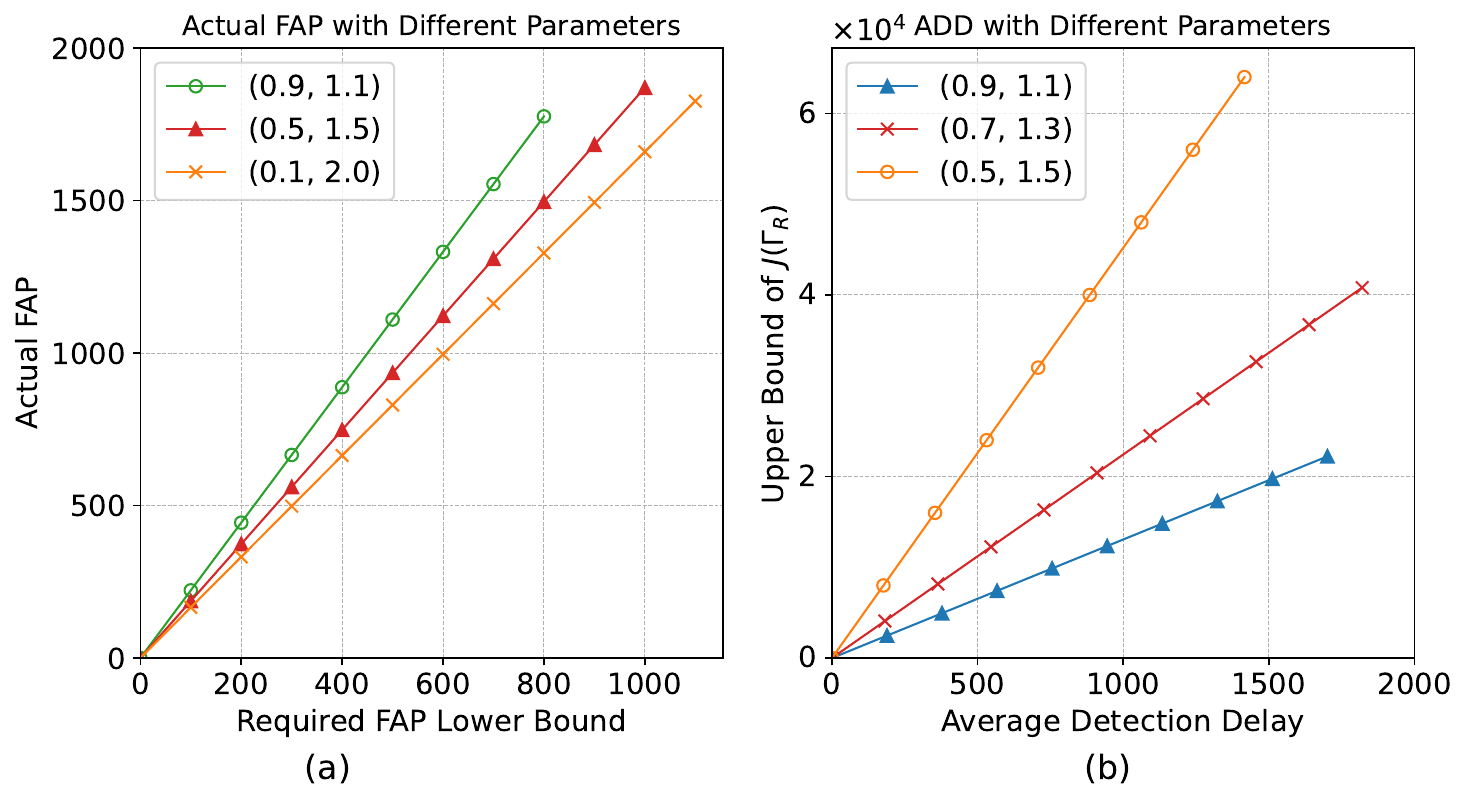} % 图片路径和大小
    \caption{Numerical Verification of Theorems}
    \label{fig:numerical_verification}
\end{figure}

\section{Conclusion}
In real-world applications, system uncertainty caused by estimation errors, model drift, and environmental disturbances presents a fundamental challenge to online change detection. 
To address this challenge, we propose RoS-Guard, a robust and scalable detection framework. Our method explicitly addresses system uncertainty. Based on Lorden’s minimax delay formulation, it achieves the minimax-optimal worst-case detection delay when solved exactly, while providing near-optimal performance in practical implementations. To enable deployment in large-scale systems, we develop a parallel algorithm leveraging neural unrolling with GPU acceleration, which significantly improves computational efficiency without sacrificing accuracy. We provide theoretical guarantees on detection performance, and extensive experiments on both synthetic and real-world datasets demonstrate the effectiveness, scalability, and robustness of the proposed approach.

\section{Acknowledgments}
 This work was supported in part by the National Natural Science Foundation of China under Grant 12371519 and 61771013; in part by Asiainfo Technologies; in part by the Fundamental Research Funds for the Central Universities of China; and in part by the Fundamental Research Funds of Shanghai Jiading District.

\bibliography{aaai2026}
\clearpage
\section{Appendix A: Theoretical Analysis}
\subsection{Generalized CUSUM Algorithm}
\label{Ad_A}

In this section, we provide the complete derivation of \eqref{v_t}. Considering the model \cref{eq:model}, we represent the probability density functions of the received signal $\mathbf{x}^{(t)}$ before and after the change as $f_{p}\left(\mathbf{x}^{(t)} \mid \boldsymbol{\theta}^{(t)},\mathbf{H}\right) $ and $f_{q}\left(\mathbf{x}^{(t)} \mid  \boldsymbol{\theta}^{(t)}, \mathbf{H}, \Delta\mathbf{H}^{(t)}\right)$, respectively.
\begin{equation}
\begin{aligned}
   &T_\text{C}=\min \{K\}\\
   &s.t.\max\limits_{1\le k\le K}\sum\limits_{t=k}^K\ln\frac{f_{q}\left(\mathbf{x}^{(t)} \mid  \boldsymbol{\theta}^{(t)}, \mathbf{H}, \Delta\mathbf{H}^{(t)}\right)}{f_{p}\left(\mathbf{x}^{(t)} \mid \boldsymbol{\theta}^{(t)},\mathbf{H}\right)}\ge h,    
\end{aligned}
\end{equation}
where $h$ is the predefined threshold. However, in MIMO systems, $\mathbf{H}$ is uncertain and $\boldsymbol{\theta}^{(t)}$, $\Delta\mathbf{H}^{(t)}$ are unknown, making the CUSUM test impractical. To address this issue, by estimating the unknown $\boldsymbol{\theta}^{(t)}$, $\Delta\mathbf{H}^{(t)}$ with their maximum likelihood estimates (MLE) \cite{GCUSUM}, we can obtain (\ref{GCUSUM}) and (\ref{test statistic}). Define 
\begin{equation}
\begin{aligned}
\delta_{k}^{(K)} \triangleq
\ln \frac{\sup \limits_{\boldsymbol{\theta}^{(t)}, \Delta\mathbf{H}^{(t)}} \prod_{t=k}^{K} f_{q}\left(\mathbf{x}^{(t)} \mid  \boldsymbol{\theta}^{(t)}, \mathbf{H}, \Delta\mathbf{H}^{(t)}\right)}
{\sup \limits_{\boldsymbol{\theta}^{(t)}} 
\prod_{t=k}^{K} 
f_{p}\left(\mathbf{x}^{(t)} \mid \boldsymbol{\theta}^{(t)},\mathbf{H}\right)
}.
\\
s.t. \quad \left\{\begin{array}{lc}
    \Delta\mathbf{H}^{(t)}:\left\{\rho_{L} \leq\left|\mu_{m}^{(t)}\right| \leq \rho_{U}\right\},m \in \mathcal{U}^{(t)}\\
    \bm{\mu}^{(t)}=\{\mu_m^{(t)}\}\in \mathcal{C}^{\perp}(\mathbf{H})
    \end{array}\right.
\end{aligned}
\end{equation}
% \begin{equation}
%    T_G  = \mathrm{min} \{  {K}:\underbrace{\operatorname*{max}_{1\leq k \leq K} \sup_{ \{\mathcal{U}^{(t)}\}, \mathbf{H}\in \mathcal{S}   } \delta_{k}^{(K)}}_{V_{K}}\geq h\},
%    \label{A.1}
% \end{equation}
% where 
% \begin{equation}
% \resizebox{0.90\linewidth}{!}{$
% \delta_{k}^{(K)} \triangleq\ln \frac{\sup \limits_{\boldsymbol{\theta}^{(t)}, \Delta\mathbf{H}^{(t)}: {\left\{\rho_{L} \leq\left|\mu_{m}^{(t)}\right| \leq \rho_{U}\right\}_{m \in \mathcal{U}^{(t)}}, \boldsymbol{\mu}^{(t)} \in \mathcal{C}^{\perp}(\mathbf{H})}} \prod_{t=1}^{k-1} f_{p}\left(\mathbf{x}^{(t)} \mid \boldsymbol{\theta}^{(t)},\mathbf{H}\right) \prod_{t=k}^{K} f_{q}\left(\mathbf{x}^{(t)} \mid  \boldsymbol{\theta}^{(t)}, \mathbf{H}, \Delta\mathbf{H}^{(t)}\right)}
% {\sup \limits_{\boldsymbol{\theta}^{(t)}} 
% \prod_{t=1}^{K} 
% f_{p}\left(\mathbf{x}^{(t)} \mid \boldsymbol{\theta}^{(t)},\mathbf{H}\right)
% }.
% $}
% \label{A.2}
% \end{equation}
We can further obtain that
\begin{equation}
\begin{aligned}
\delta_{k}^{(K)} =  &
\sum_{t=k}^{K}   \bigg\{ \sup \limits_{\boldsymbol{\theta}^{(t)}, \Delta\mathbf{H}^{(t)}}
f_{q}\left(\mathbf{x}^{(t)} \mid  \boldsymbol{\theta}^{(t)}, \mathbf{H}, \Delta\mathbf{H}^{(t)}\right)
\nonumber  \\ &\qquad \quad  -\sup \limits_{\boldsymbol{\theta}^{(t)}}  \ln 
f_{p}\left(\mathbf{x}^{(t)} \mid \boldsymbol{\theta}^{(t)},\mathbf{H}\right) 
\bigg\} \\
  \triangleq & \sum_{t=k}^{K} \delta_{k,t}^{(K)}.\\
s.t. \quad &\left\{\begin{array}{lc}
    \Delta\mathbf{H}^{(t)}:\left\{\rho_{L} \leq\left|\mu_{m}^{(t)}\right| \leq \rho_{U}\right\},m \in \mathcal{U}^{(t)}\\
    \bm{\mu}^{(t)}=\{\mu_m^{(t)}\}\in \mathcal{C}^{\perp}(\mathbf{H})
    \end{array}\right.
\label{A.3}
\end{aligned}
\end{equation}

Given the model \eqref{eq:model} and the Gaussian distribution of $\mathbf{n}^{(t)}$, we have that
\begin{equation}
\begin{aligned}
&f_{p}\left(\mathbf{x}^{(t)} \mid \boldsymbol{\theta}^{(t)},\mathbf{H}\right) \\
 =& \frac{\exp\bigg[-\frac{\left(\mathbf{x}^{(t)}-\mathbf{H}\boldsymbol{\theta}^{(t)}\right)^T\left(\mathbf{x}^{(t)}-\mathbf{H}\boldsymbol{\theta}^{(t)}\right)}{2\sigma_n^2}   \Bigg]}{\left(2\pi\sigma_n^2\right)^{\frac{M}{2}}},
\label{A.4}
\end{aligned}
\end{equation}
\begin{equation}
\begin{aligned}
& f_{q}\left(\mathbf{x}^{(t)} \mid  \boldsymbol{\theta}^{(t)}, \mathbf{H},  \Delta\mathbf{H}^{(t)}\right) \\
=& \frac{\exp\bigg[-\frac{\left(\mathbf{x}^{(t)}-\mathbf{H}\boldsymbol{\theta}^{(t)}-\Delta\mathbf{H}^{(t)}\boldsymbol{\theta}^{(t)}\right)^T\left(\mathbf{x}^{(t)}-\mathbf{H}\boldsymbol{\theta}^{(t)}-\Delta\mathbf{H}^{(t)}\boldsymbol{\theta}^{(t)}\right)}{2\sigma_n^2} \Bigg]}{\left(2\pi\sigma_n^2\right)^{\frac{M}{2}}}.
\label{A.5}
\end{aligned}
\end{equation}
Let $\tilde{\mathbf{x}}^{(t)}$ represents the component of $\mathbf{x}^{(t)}$ orthogonal to the column space of $\mathbf{H}$
i.e., 
\begin{equation}
\begin{aligned}
    \tilde{\mathbf{x}}^{(t)}=\mathbf{P}_{\mathbf{H}}^{\perp} \mathbf{x}^{(t)}
\label{comple_space_H}
\end{aligned}
\end{equation}
It then follows from (\ref{A.4}) and (\ref{A.5}) that,
\begin{equation}
\begin{aligned}
& \sup \limits_{\boldsymbol{\theta}^{(t)}}  \ln 
f_{p}\left(\mathbf{x}^{(t)} \mid \boldsymbol{\theta}^{(t)},\mathbf{H}\right) 
= -\frac{\|\tilde{\mathbf{x}}^{(t)}\|^{2}}{2\sigma_n^{2}} -  \frac{M\ln( 2\pi\sigma_n^2)}{2} ,
\label{A.6} \\
\end{aligned}
\end{equation}

\begin{equation}
    \begin{aligned}
        & \sup \limits_{\boldsymbol{\theta}^{(t)}, \Delta\mathbf{H}^{(t)}
} \ln f_{q}\left(\mathbf{x}^{(t)} \mid  \boldsymbol{\theta}^{(t)}, \mathbf{H}, \Delta\mathbf{H}^{(t)}\right)\\
=& -  \frac{M\ln( 2\pi\sigma_n^2)}{2}    
+\sup\limits_{
{\boldsymbol{\mu}^{(t)}}
} \Bigg[ -\frac{\left(\tilde{\mathbf{x}}^{(t)} - \boldsymbol{\mu}^{(t)}\right)^T\left(\tilde{\mathbf{x}}^{(t)}-\boldsymbol{\mu}^{(t)}\right)}{2\sigma_n^2} 
\Biggr] 
\label{A.7}
    \end{aligned}
\end{equation}
By combining (\ref{A.3}) and (\ref{A.6}) with (\ref{A.7}), we have that
\begin{equation}
  \begin{aligned}
\delta_{k,t}^{(K)}
& =  \sup\limits_{\boldsymbol{\mu}^{(t)}}    \frac{1}{2\sigma_n^2} \left\{
2(\boldsymbol{\mu}^{(t)})^{T}\tilde{\mathbf{x}}^{(t)} - \| \boldsymbol{\mu}^{(t)}\|_{2}^{2} 
\right\}  .
\label{A.8}
\end{aligned}  
\end{equation}
Thus, by employing (\ref{A.8}), the decision statistic $V_K$ can be rewritten as
\begin{equation}
\begin{aligned}
V_{K}& \stackrel{\Delta}{=}\max_{1\leq k\leq K}\sup_{\{\mathcal{U}^{(t)}\}, \mathbf{H}\in \mathcal{S}  }\delta_{k}^{(K)} =\max_{1\leq k\leq K}\sup_{\{\mathcal{U}^{(t)}\},\mathbf{H}\in \mathcal{S}  }\sum_{t=k}^{K}\delta_{k,t}^{(K)} \\
&=\max_{1\leq k\leq K}\sum_{t=k}^{K}v_{t},
\end{aligned}
\end{equation}
where
\begin{equation}
\begin{aligned}
    &v_{t} = \sup _{\mathcal{U}^{(t)}}   \sup_{ 
    \substack{ \boldsymbol{\mu}^{(t)}
    } } \frac{1}{2 \sigma^{2}} \times  \left\{  2\left(\boldsymbol{\mu}^{(t)}\right)^{T} \tilde{\mathbf{x}}^{(t)}-\left\|\boldsymbol{\mu}^{(t)}\right\|_{2}^{2} \right\} ,\\
    &s.t. \quad \left\{\begin{array}{lc}
    \Delta\mathbf{H}^{(t)}:\left\{\rho_{L} \leq\left|\mu_{m}^{(t)}\right| \leq \rho_{U}\right\},m \in \mathcal{U}^{(t)}\\
    \bm{\mu}^{(t)}=\{\mu_m^{(t)}\}\in \mathcal{C}^{\perp}(\mathbf{H})
    \end{array}\right.
\end{aligned}
\label{ap_vt}
\end{equation}
Since $\tilde{\mathbf{x}}^{(t)}$ is orthogonal to the column space of $\mathbf{H}$ (\ref{ap_vt}) is equivalent to (\ref{v_t})

\subsection{Derivation of Reformulation \eqref{PC_formulation}}
\paragraph{Reformulation via Strong Duality:}

Consider the robust constraint under a polyhedral uncertainty set:
\begin{equation}
\bar{\mathbf{h}}_i^T \boldsymbol{\mu}_i + \max_{\mathbf{h}_i \in \mathcal{S}_i} (\mathbf{h}_i - \bar{\mathbf{h}}_i)^T \boldsymbol{\mu}_i \leq \varepsilon_i, \\
\mathcal{S}_i = \left\{ \mathbf{h}_i \mid \mathbf{D}_i \mathbf{h}_i \leq \mathbf{d}_i \right\}.
\label{eq:original_constraint}
\end{equation}
The inner maximization problem in \eqref{eq:original_constraint} can be rewritten as
\begin{equation}
\max_{\mathbf{h}_i \in \mathcal{S}_i} (\mathbf{h}_i - \bar{\mathbf{h}}_i)^T \boldsymbol{\mu}_i = \max_{\mathbf{h}_i \in \mathcal{S}_i} \mathbf{h}_i^T \boldsymbol{\mu}_i - \bar{\mathbf{h}}_i^T \boldsymbol{\mu}_i.
\label{eq:inner_max}
\end{equation}
Substituting \eqref{eq:inner_max} into \eqref{eq:original_constraint}, the robust constraint becomes
\begin{equation}
\label{eq:max_problem}
\begin{aligned}
  \max_{\mathbf{h}_i }&\quad \mathbf{h}_i^T \boldsymbol{\mu}_i \leq \varepsilon_i\\
\text{s.t.}& \quad \mathbf{D}_i \mathbf{h}_i \leq \mathbf{d}_i
\end{aligned}
\end{equation}
Using strong duality, the maximization problem \eqref{eq:max_problem} can be equivalently reformulated as the following dual minimization problem:
\begin{equation}
\label{eq:max_problem}
\begin{aligned}
  \min_{\mathbf{p}_i\geq \bm{0}}&\quad \mathbf{p}_i^T \mathbf{d}_i \\
\text{s.t.}& \quad \mathbf{D}_i^T \mathbf{p}_i = \boldsymbol{\mu}_i.
\end{aligned}
\end{equation}
Therefore, the robust constraint \eqref{eq:original_constraint} is equivalent to
\begin{equation}
\min_{\mathbf{p}_i}\mathbf{p}_i^T \mathbf{d}_i \leq \varepsilon_i, \quad \mathbf{D}_i^T \mathbf{p}_i = \boldsymbol{\mu}_i, \quad \mathbf{p}_i \geq \bm{0}.
\label{eq:reformulated_constraint}
\end{equation}
Since the objective is to minimize $\boldsymbol{\mu}_i$, and $\boldsymbol{p}_i$ is linearly related to $\boldsymbol{\mu}_i$ through the constraint $\mathbf{D}_i^T \boldsymbol{p}_i = \boldsymbol{\mu}_i$, the minimization over $\boldsymbol{p}_i$ in constraint \eqref{eq:reformulated_constraint} is implicitly handled by the objective. Therefore, the min operator in the constraint can be removed, and it suffices to enforce $\boldsymbol{p}_i^T \mathbf{d}_i \leq \varepsilon_i$ directly.

\paragraph{Reformulation with Epigraph:}
Consider the one-dimensianal case of MIQP objective function as:
\begin{equation}
    \frac{1}{2 \sigma^{2}}  \left\{ \left\|(\mu_m^{+} - \mu_m^{-})\right\|_{2}^{2} -2 {(\mu_m^{+} - \mu_m^{-}) }^{T}{x_m} \right\}.
\end{equation}
We consider reformulating the nonlinear term in the above expression using the epigraph approach, transforming it into a linear term and obtaining an equivalent convex cone constraint. Specifically, we first introduce \( u_m \) to represent the mutual exclusivity between \(\mu_m^+\) and \(\mu_m^-\). Then, by introducing the variable \(\phi_m\) satisfying
$
u_m^{-1} (\mu_m^+ - \mu_m^-)^2 \leq \phi_m,
$
to replace the nonlinear term in the objective function.

Here, we prove that the inequality $ u_m^{-1} (\mu_m^+ - \mu_m^-)^2\leq \phi_m$ is equivalent to the second-order cone constraint
\begin{equation}
    \begin{Vmatrix} {\mu}_m^{+} - {\mu}_m^{-} \\\frac{\phi_m-u_m}{2}\end{Vmatrix} \leq           \frac{\phi_m+u_m}{2}.
    \label{eq:soc}
\end{equation}
Expanding the left-hand side of the \eqref{eq:soc} yields
\begin{equation}
\left\|
\begin{bmatrix}
\mu_m^+ - \mu_m^- \\
\frac{\phi_m - u_m}{2}
\end{bmatrix}
\right\|
= \sqrt{ (\mu_m^+ - \mu_m^-)^2 + \left( \frac{\phi_m - u_m}{2} \right)^2 }.
\end{equation}
Thus, by simultaneously squaring both sides of \eqref{eq:soc}, it follows
\begin{equation}
    (\mu_m^+ - \mu_m^-)^2 + \left( \frac{\phi_m - u_m}{2} \right)^2\leq\left( \frac{\phi_m + u_m}{2} \right)^2.
\end{equation}
By further eliminating terms on both sides of the inequality and dividing by $u_m$ we obtain 
\begin{equation}
    \frac{(\mu_m^+ - \mu_m^-)^2}{u_m} \leq \phi_m.
\end{equation}

\subsection{Relaxation (\ref{PC_formulation}) versus SDP relaxation}
\label{app-B}

In this section, we present a proof demonstrating that the relaxation (\ref{PC_formulation}) offers a tighter bound compared to the traditional SDP relaxation.
We introduce auxiliary matrices $\mathbf{U}$ and $\mathbf{B}$ for the integer variables $\mathbf{u}$ and $\mathbf{b}$ in problem (\ref{MIQP}), respectively. Since both $\mathbf{u}$ and $\mathbf{b}$ are 0-1 integer vectors, we have that, 
\begin{equation}
\begin{aligned}
&\mathbf{U} = \mathbf{u}\mathbf{u}^{T}  \\
&\mathbf{U}_{mm} = u_m, \forall m.
\end{aligned}
\end{equation}
The SDP relaxtion is obtained by replacing the nonconvex equality constraint  $\mathbf{U} = \mathbf{u}\mathbf{u}^{T}$ with a positive semi-definite constraint $\mathbf{U} - \mathbf{u}\mathbf{u}^{T} \succeq 0$. This constraint can be formulated as follows by employing Schur complement.
\begin{equation}
\begin{bmatrix}U&u\\u^T&1\end{bmatrix}\succeq0
\end{equation}
The previous procedure applies equally to variable $\mathbf{b}$.
According to \cite{}, we have that,
\begin{equation}
\begin{cases}
U_{mm}=u_m,m=1,\ldots,n\\\begin{bmatrix} \mathbf{U} 
& \mathbf{u}\\
\mathbf{u}^T &  1  \end{bmatrix}\succeq 0
\end{cases}  \iff 0 \leq u_m \leq 1, \forall m.
\end{equation}
Since the matrices $\mathbf{U}$ and $\mathbf{B}$ are not involved in any other constraints or the objective function of (\ref{PC_formulation}), the SDP relaxation is equivalent to simply relaxing $\mathbf{u}$ and $\mathbf{b}$ to $[0,1]^m$.
Recall from the Problem Relexation section that relaxation (\ref{PC_formulation}) is attained by minimizing the convex envelope of the function $g$ across all feasible points. As a result, it provides tighter bounds compared to the SDP relaxation.

\subsection{Proof of \cref{theorem_2}}
\label{Poof_2}
In this section, we provide the complete proof of \cref{theorem_2}.
\begin{proof}
Due to the constraint $\boldsymbol{\mu}^{(t)} \in \mathcal{C}^{\perp}(\mathbf{H})$, 
problem \eqref{v_t}) is equivalent to 
\begin{equation}
\label{new_v_t}
\begin{split}
    v_{t} =- \min _{\mathcal{U}^{(t)}}  & \min_{\boldsymbol{\mu}^{(t)}:\left\{\rho_{L} \leq\left|\mu_{m}^{(t)}\right| \leq \rho_{U}\right\}_{m \in \mathcal{U}^{(t)}}, \boldsymbol{\mu}^{(t)} \in \mathcal{C}^{\perp}(\mathbf{H})} k(\boldsymbol{\mu}^{(t)}), 
\end{split}
\end{equation}
where $k(\boldsymbol{\mu}^{(t)}) = \frac{1}{2 \sigma^{2}} 
\big\{   \big\|\boldsymbol{\mu}^{(t)}\big\|_{2}^{2} - 2\left(\boldsymbol{\mu}^{(t)}\right)^{T} \tilde{\boldsymbol{x}}^{(t)}
\big\}$.
Thus, we have the upper bound of $v_{t}$ as follows,
\begin{equation}
\label{C-1}
 v_{t} \leq  {\frac{1}{2 \sigma^{2}} }||{\tilde{\boldsymbol{x}}^ {(t)} ||}_{2}^2.
\end{equation}
By combining (\ref{V_K}) with (\ref{C-1}), the upper bound of $\mathbb{E}_{\infty}\left\{V_{\Gamma_R}\right\}$ can be obtained as follows,
\begin{equation}
\begin{aligned}
\label{C-2}
\mathbb{E}_{\infty}\left\{V_{\Gamma_R}\right\} 
& = \mathbb{E}_{\infty}\left\{\max _{1 \leq k \leq \Gamma_R} \sum_{t=k}^{\Gamma_R} v_{t}\right\}    \\
& \leq  \mathbb{E}_{\infty}\left\{\max _{1 \leq k \leq \Gamma_R} \sum_{t=k}^{\Gamma_R} \frac{1}{2 \sigma^{2}}{||\tilde{\boldsymbol{x}}^ {(t)} ||}_{2}^2  \right\}   \\
&  \leq \mathbb{E}_{\infty}\left\{\sum_{t=1}^{\Gamma_R} {\frac{1}{2 \sigma^{2}}} 
 {||\tilde{\boldsymbol{x}}^ {(t)} ||}_{2}^2     \right\} \\
& =  \mathbb{E}_{\infty}\left\{\sum_{t=1}^{\infty} {\frac{1}{2 \sigma^{2}} }||{\tilde{\boldsymbol{x}}^ {(t)} ||}_{2}^2  ~ \mathbb{I}\left\{\Gamma_R \geq t\right\}\right\}.
\end{aligned}
\end{equation}
According to the monotone convergence theorem and the independence between  $|{|\tilde{\boldsymbol{x}}^ {(t)} ||}^2$ and $\mathbb{I} \{\Gamma_R \geq t\}$, we have that,
\begin{equation}
\begin{aligned}
\mathbb{E}_{\infty}\left\{V_{\Gamma_R}\right\} 
& \leq \sum_{t=1}^{\infty} \mathbb{E}_{\infty}\left\{ \frac{1}{2 \sigma^{2}}||{\tilde{\boldsymbol{x}}^ {(t)} ||}_{2}^2   ~ \mathbb{I}\left\{\Gamma_R \geq t\right\}\right\}  \\
& = \sum_{t=1}^{\infty}  \mathbb{E}_{\infty}\left\{ {\frac{1}{2 \sigma^{2}}||\tilde{\boldsymbol{x}}^ {(t)} ||}_{2}^2  \right\} \mathbb{E}_{\infty}\left\{\mathbb{I}\left\{\Gamma_R \geq t\right\}\right\}  \\
& = \sum_{t=1}^{\infty} \mathbb{E}_{\infty}\left\{\frac{1}{2 \sigma^{2}}|| {\tilde{\boldsymbol{x}}^ {(t)} || }_{2}^2  \right\} \mathbb{P}_{\infty}\left(\Gamma_R \geq t\right), \label{C-5}
\end{aligned}
\end{equation}
where $\mathbb{P}_{\infty}$ is the probability measure when no change occurs. Since $\| \mathbf{x}^{(t)} \|_2^2 \leq \alpha$ and $\tilde{\boldsymbol{x}}^ {(t)} $ is a component of $\mathbf{x}^{(t)} $, we can further obtain that, 
\begin{align}
\label{C-6}
\mathbb{E}_{\infty}\left\{V_{\Gamma_R}\right\} 
\leq & \frac{\alpha}{2\sigma^2} \sum_{t=1}^{\infty}\mathbb{P}_{\infty}\left(\Gamma_R \geq t\right) \nonumber  \\
= &\frac{\alpha}{2\sigma^2}  \mathbb{E}_{\infty}\{  \Gamma_R \}.
\end{align}
When the false alarm is declared, we have $V^{(\Gamma_R)} \geq h$, which implies
\begin{equation}
\label{C-7}
\mathbb{E}_{\infty}\left\{V^{\left(\Gamma_R\right)}\right\} \geq h,
\end{equation}
and therefore by employing (\ref{C-6}), we can obtain that, 
\begin{equation}
\label{C-8}
    \mathbb{E}_{\infty}\{  \Gamma_R \} \geq \frac{2h\sigma^2}{\alpha}.
\end{equation}
Thus, if
\begin{equation}
\label{C-9}
h \geq \frac{\alpha}{2\sigma^2}\gamma,
\end{equation}
then $\mathbb{E}_{\infty}\{  \Gamma_R \} \geq \gamma$ is guaranteed, which completes the
proof.
\end{proof}

\subsection{Proof of \cref{theorem_3}}
% \section{Proof of \texorpdfstring{\cref{theorem_3}}{Theorem 3}}
\label{Poof_3}
In this section, we provide the complete proof of \cref{theorem_3}.
In order to prove \cref{theorem_3}, We will start by introducing two lemmas. \cref{lemma1} provides an upper bound on $\mathbb{E}_{t_a}\left\{\left.(\Gamma_R-t_a+1)^+\right|\mathcal{F}_{t_a-1}\right\}$. And \cref{lemma2} specifies that the stopping time $\Gamma_R$ utilized in SDPCUSUM and BBCUSUM achieves the equalizer rule. On the basis of these two lemmas, we finally provide the proof of \cref{theorem_3}.

\begin{lemma}
\label{lemma1}
For any given $t_a$, we have that, 
\begin{equation}
\begin{aligned}
&\mathbb{E}_{t_{a}}\bigg\{(\Gamma_R-t_{a}+1)^{+}  \bigg|\mathcal{F}_{t_{a}-1} \bigg\}\\
\leq& \frac{2\sigma^2}{{\rho_L}^2}  \mathbb{E}_{t_a} \bigg
\{V_{\Gamma_R}\mathds{1}\left\{\Gamma_R\geq t_a\right\}\bigg|\mathcal{F}_{t_a-1} \bigg\}.
\end{aligned}
\end{equation}
\end{lemma}
\begin{proof}
Let $\boldsymbol{h}_{t_i}$ represents the true value of $\boldsymbol{h}_{i}$.
Since $d_i$ denotes the diameter of uncertainty set and $\rho_H$ is the upper bound of every element of $ \boldsymbol{\mu}^{(t)}$, we have that, 
\begin{equation}
(\boldsymbol{h}_i^T - \boldsymbol{h}_{t_i}^T) \boldsymbol{\mu}^{(t)} \leq 
\| \boldsymbol{h}_i^T - \boldsymbol{h}_{t_i}^T\| \| \boldsymbol{\mu}^{(t)} \| \leq 
d_i \rho_H M^{\frac12}, \forall \boldsymbol{h}_{i} \in \mathcal{U}_{i}
\end{equation}
Since $\varepsilon_{i} \geq d_i \rho_H M^{\frac12}$, we can further obtain that, 
\begin{equation}
-\varepsilon_{i} \leq (\boldsymbol{h}_i^T - \boldsymbol{h}_{t_i}^T) \boldsymbol{\mu}^{(t)} \leq \varepsilon_{i}, \forall \boldsymbol{h}_{i} \in \mathcal{U}_{i}.
\end{equation}
Since $ \boldsymbol{h}_{t_i}^T \boldsymbol{\mu}^{(t)} =0$, we have that,
\begin{equation}
-\varepsilon_{i} \leq \boldsymbol{h}_i^T \boldsymbol{\mu}^{(t)} \leq \varepsilon_{i}, \forall \boldsymbol{h}_{i} \in \mathcal{U}_{i}
\label{ro_1}
\end{equation}
which indicates that every $ \boldsymbol{\mu}^{(t)}$ satisfies $ \boldsymbol{h}_{t_i}^T \boldsymbol{\mu}^{(t)} =0$ also satisfies (\ref{ro_1}).

It can be seen from \eqref{eq:model} and (\ref{comple_space_H}) that 
\begin{equation}
\mathbb{E}_{t_{a}} \{ 
\tilde{\boldsymbol{x}}^{(t)}    \}= \boldsymbol{\mu}^{(t)}.
\end{equation}
By employing  (\ref{new_v_t}), when $t \geq t_a$,  we have that,
\begin{equation}
\mathbb{E}_{t_{a}}\left\{ v_t\right\}  \geq \mathbb{E}_{t_{a}} \{ -f( \boldsymbol{\mu}^{(t)}   )  \} =  \frac{1}{2\sigma^2} \| \boldsymbol{\mu}^{(t)} \|_2^2 \geq \frac{ {\rho_L}^2 }{2\sigma^2}.
\label{ne_4}
\end{equation}
Since $\mathds{1} \{\Gamma_R\geq t_{a}\}$ is $\mathcal{F}_{t_a-1}$-measurable, by applying the monotone convergence theorem, we have 
\begin{equation}
\begin{aligned}
& \mathbb{E}_{t_a} \bigg\{\mathds{1} \{\Gamma_R\geq t_a \}\sum_{t=t_a}^{\Gamma_R}v_t \bigg|\mathcal{F}_{t_a-1}\bigg\}  \\
& =   \mathds{1} \{\Gamma_R\geq t_{a}\} \times\mathbb{E}_{t_{a}} \bigg \{\sum_{t=t_{a}}^{\infty}v_t \mathds{1}\{\Gamma_R\geq t \}\bigg|\mathcal{F}_{t_{a}-1}\bigg\} \\
&=\mathds{1} \{\Gamma_R\geq t_{a}\} \times\sum_{t=t_{a}}^{\infty}\mathbb{E}_{t_{a}}\bigg\{ v_t \mathds{1} \{\Gamma_R\geq t \} \bigg| \mathcal{F}_{t_{a}-1} \bigg \}.  \label{d6}
\end{aligned}
\end{equation}
And we further have  
\begin{align}
& \mathds{1} \{\Gamma_R\geq t_{a}\} \times\sum_{t=t_{a}}^{\infty}\mathbb{E}_{t_{a}}\bigg\{ v_t \mathds{1} \{\Gamma_R\geq t \} \bigg| \mathcal{F}_{t_{a}-1} \bigg \} \nonumber \\
& = \mathds{1} \{\Gamma_R\geq t_{a}\} \times\sum_{t=t_a}^\infty\mathbb{E}_{t_a} \bigg\{\mathbb{E}_{t_a} \bigg\{v_t \mathds{1} \{\Gamma_R\geq t \} \bigg| \mathcal{F}_{t-1} \bigg\}\bigg|\mathcal{F}_{t_a-1} \bigg\}  \label{d7}\\
& = \mathds{1} \{\Gamma_R\geq t_{a}\} \times  \sum_{t=t_a}^\infty\mathbb{E}_{t_a} \bigg\{\mathds{1} \{\Gamma_R\geq t \} \times  \mathbb{E}_{t_a}\big\{v_t\big\} \bigg|\mathcal{F}_{t_a-1}\bigg\}, \label{d8}
\end{align}
where (\ref{d7}) is obtained by employing the Tower’s property and $\mathcal{F}_{t_a-1}\subseteq\mathcal{F}_{t-1}$ when $t \geq t_a$. (\ref{d8}) comes from that fact that $\mathds{1} \{\Gamma_R\geq t_{a}\}$ is $\mathcal{F}_{t-1}$-measurable.
By combining (\ref{d6}) and (\ref{d8}) with (\ref{ne_4}), we have
\begin{equation}
\begin{aligned}
& \mathbb{E}_{t_a} \bigg\{ \mathds{1} \{\Gamma_R\geq t_a \}\sum_{t=t_a}^{\Gamma_R} v_t \bigg|\mathcal{F}_{t_a-1} \bigg \} \\
 \geq & \frac{ {\rho_L}^2 }{2\sigma^2} 
\times \mathds{1} \{\Gamma_R\geq t_{a} \}\sum_{t=t_{a}}^{\infty}\mathbb{E}_{t_{a}} \bigg\{ \mathds{1} \{\Gamma_R\geq t \} \bigg|\mathcal{F}_{t_{a}-1} \bigg\}.
\label{d9}  
\end{aligned}
\end{equation}

Since  $\mathds{1} \{\Gamma_R\geq t_a \}$ is $\mathcal{F}_{t_a -1}$-measurable, by employing the monotone convergence theorem, we can obtain
\begin{equation}
\begin{aligned}
&  \mathds{1} \{\Gamma_R\geq t_{a} \}\sum_{t=t_{a}}^{\infty}\mathbb{E}_{t_{a}} \bigg\{  \mathds{1} \{\Gamma_R\geq t \} \bigg|\mathcal{F}_{t_{a}-1} \bigg\} \\
&  = \mathbb{E}_{t_{a}}  \bigg\{  \mathds{1} \{\Gamma_R\geq t \} \sum_{t=1}^{\infty} \mathds{1} \{ (\Gamma_R - t_a + 1) \geq t \}
\bigg|\mathcal{F}_{t_{a}-1} \bigg\} \\
& =  \mathbb{E}_{t_a}  \bigg\{ (\Gamma_R-t_a+1) \mathds{1}\{\Gamma_R\geq t_a \} \bigg|\mathcal{F}_{t_a-1} \bigg\}   \\
&  =  \mathbb{E}_{t_{a}} \bigg\{ (\Gamma_R-t_{a}+1)^{+} \bigg|\mathcal{F}_{t_{a}-1} \bigg\}. 
\label{d11} 
\end{aligned}
\end{equation}
By combining (\ref{d9}) with (\ref{d11}), we have 
\begin{equation}
\begin{aligned}
&\mathbb{E}_{t_a} \bigg\{ \mathds{1} \{\Gamma_R\geq t_a \}\sum_{t=t_a}^{\Gamma_R} v_t \bigg|\mathcal{F}_{t_a-1} \bigg \} \\
\geq & \frac{ {\rho_L}^2 }{2\sigma^2} 
\times \mathbb{E}_{t_{a}} \bigg\{ (\Gamma_R-t_{a}+1)^{+} \bigg|\mathcal{F}_{t_{a}-1} \bigg\}.
\label{nd1}
\end{aligned}
\end{equation}
Since $v_t \geq 0 $ according to \eqref{v_t}, by employing (\ref{V_K}), we have
\begin{equation}
\begin{aligned}
&V_{\Gamma_R}\mathds{1}\left\{\Gamma_R\geq t_{a}\right\}\\
 =&\mathds{1}\left\{\Gamma_R\geq t_a\right\}\sum_{t=1}^{\Gamma_R}v_t \\
=&\mathds{1}\left\{\Gamma_R\geq t_a\right\}\sum_{t=1}^{t_a-1}v_t+ \mathds{1}\{\Gamma_R\geq t_{a}\}\sum_{t=t_{a}}^{\Gamma_R} v_t\\
\geq & \mathds{1}\{\Gamma_R\geq t_{a}\}\sum_{t=t_{a}}^{\Gamma_R} v_t,
\end{aligned}
\label{d12} 
\end{equation}
Therefore, based on  (\ref{nd1}) and (\ref{d12}), we have that,
\begin{equation}
\begin{aligned}
&\mathbb{E}_{t_{a}}\left\{(\Gamma_R-t_{a}+1)^{+}\bigg|\mathcal{F}_{t_{a}-1}\right\} \\
\leq &\frac{2\sigma^2}{{\rho_L}^2}  \mathbb{E}_{t_a}\bigg\{\left.V_{\Gamma_R}\mathds{1}\left\{\Gamma_R\geq t_a\right\}\bigg|\mathcal{F}_{t_a-1}\right\}.
\end{aligned}
\end{equation}
which completes the proof for \cref{lemma1}.
\end{proof}
 
\begin{lemma}
\label{lemma2}
The stopping time $\Gamma_R$ of SDPCUSUM and BBCUSUM achieves an equalizer rule, i.e.,
\begin{equation}
\begin{aligned}
J\left(\Gamma_R\right)&=\sup_{t_a}J_{t_a}\left(\Gamma_R\right)=J_1\left(\Gamma_R\right).
\end{aligned}
\end{equation}
\end{lemma}
\begin{proof}
For any $K_1 \leq K$, according to (\ref{V_K}) and \eqref{v_t}, we can obtain that,
\begin{equation}
V_{K}
=V_{K_1}+\sum_{t=K_1+1}^{K}v_t.
\label{newadd_1}
\end{equation}
It is seen from (\ref{newadd_1}) that given $ \{\mathbf{x}^{{K_1 + 1}},\mathbf{x}^{{K_1 + 2}}, \cdots, \mathbf{x}^{K} \}$, $V_{K}$ increases as $V_{K_1}$ increases.
Therefore, according to (\ref{GCUSUM}), $\Gamma_R$ decreases as $V_{t_a-1}$ increases considering the event $\{\Gamma_R \leq t_a\}$. 
Given that $V_{t_a - 1} \geq 0$ and the event $\{ V_{t_a -1 } = 0 \in   \mathcal{F}_{t_a-1} \}$, the maximum value of $\mathbb{E}_{t_a}\{(\Gamma_R-t_a+1)^+|\mathcal{F}_{t_a-1}\}$ over $\mathcal{F}_{t_a-1}$ is attained at $V_{t_a - 1} = 0$,
that is,
\begin{equation}
\begin{aligned}
J_{t_{a}}\left(\Gamma_R\right)& =\operatorname{ess}\operatorname*{sup}_{\mathcal{F}_{t_{a}-1}}\mathbb{E}_{t_{a}}\bigg\{(\Gamma_R-t_{a}+1)^{+}\bigg|\mathcal{F}_{t_{a}-1}\bigg\}  \\
&=\mathbb{E}_{t_a}\bigg\{\Gamma_R-t_a+1\bigg|V_{t_a-1}=0\bigg\}.
\label{hhhhh}
\end{aligned}
\end{equation}
Since $v_t \geq 0 $ according to \eqref{v_t}, it follows from (\ref{V_K}) that the event $V_{t_a -1 } = 0$ is equivalent to the event $\{ v_t=0, \forall t=1,2,...t_a-1 \}$, which implies that,
\begin{equation}
J_{t_{a}}\left(\Gamma_R\right)
=\mathbb{E}_{t_a}\bigg \{\Gamma_R-t_a+1 \bigg| v_t=0,  \forall t=1,2,\cdots,t_a-1 \bigg \}.
\label{new_2}
\end{equation}
According to \eqref{eq:model} and (\ref{comple_space_H}), given two distinct change time instants $t_a$
and $t_a^\prime$, the respective $\tilde{x}_m^{(t_a+t)}$ and $\tilde{x}_m^{(t_a^\prime+t)}$ are identically distributed for any $t \geq 0$. As a result, from (\ref{v_t}) that $v_{t_a + t}$ and $v_{t_a^\prime+t}$ also follow the same distribution for any $t \geq 0$. Therefore, by combining (\ref{new_2}) and (\ref{V_K}),  we have, 
\begin{equation}
J_{t_a}\left(\Gamma_R\right)=J_1\left(\Gamma_R\right),\forall t_a\geq1.
\end{equation}
And therefore,
\begin{equation}
J\left(\Gamma_R\right)=\sup_{t_{a}}J_{t_{a}}\left(\Gamma_R\right)=J_{1}\left(\Gamma_R\right),
\end{equation}
which completes the proof.
\end{proof}
Finally, on the basis of above two lemmas, we provide the proof of \cref{theorem_3}. 
\begin{proof}
According to \eqref{eq:delay}, the worst-case expected detection delay of the proposed methods can be written as 
\begin{equation}
J\left(\Gamma_{R}\right)\triangleq\sup_{t_{a}}J_{t_{a}}\left(\Gamma_{R}\right),
\end{equation}
where
\begin{equation}
J_{t_a}\left(T_R\right)\triangleq\text{ess}\sup_{\mathcal{F}_{t_a-1}}\mathbb{E}_{t_a}\left\{\left(T_R-t_a+1\right)^+\bigg|\mathcal{F}_{t_a-1}\right\}. 
\label{ne_1}
\end{equation}
By combining (\ref{ne_1}) and \cref{lemma1}, we have
\begin{equation}
J_{t_a}\left(T_R\right)  \leq  \frac{2\sigma^2}{{\rho_L}^2}  \mathbb{E}_{t_a}\left\{V_{\Gamma_R} {1}\left\{\Gamma_R\geq t_a\right\}\bigg|\mathcal{F}_{t_a-1}\right\}.
\label{ne_2}
\end{equation}
Denote $\Delta$ as the overshoot, at time instant $\Gamma_R$, we have $V_{\Gamma_R}= h + \Delta$.
% \begin{equation}
% V_{\Gamma_R}= h + \Delta.
% \end{equation}
Therefore, for any given $t_a$, we can obtain,
\begin{equation}
V_{\Gamma_R}\mathds{1}\bigg\{\Gamma_R\geq t_a\bigg\}=(h+\Delta)\mathds{1}\bigg\{\Gamma_R\geq t_a\bigg\}\leq h+\Delta.
\label{new_4}
\end{equation}
Since the stopping time $\Gamma_R$ of SDPCUSUM and BBCUSUM achieves the equalizer rule as demonstrated in \cref{lemma2}, we have
\begin{equation}
\begin{aligned}
&\mathbb{E}_{t_{a}}\bigg\{ V_{\Gamma_R}\mathds{1}\left\{\Gamma_R\geq t_{a}\right\}\bigg|\mathcal{F}_{t_{a}-1}\bigg\}\\
=&\mathbb{E}_1\bigg\{V_{\Gamma_R}\mathds{1}\left\{\Gamma_R\geq1\right\}\bigg\} \leq h+\mathbb{E}_1\left\{\Delta\right\} .
\end{aligned}
\end{equation}
According to Wald’s approximations \cite{tartakovsky2014sequential}, 
the expectation of the overshoots can be ignored when change occurs. As a result, we have 
\begin{equation}
\mathbb{E}_{t_{a}}\bigg\{ V_{\Gamma_R}\mathds{1}\left\{\Gamma_R\geq t_{a}\right\} 
 \bigg|\mathcal{F}_{t_{a}-1}\bigg\}  \approx h  .
\label{ne_3}
\end{equation}
By combining (\ref{ne_2}), (\ref{ne_3}) with  \cref{lemma2} , we can obtain that,
\begin{equation}
J\left(\Gamma_R\right)=J_1\left(\Gamma_R\right) \leq \frac{2 h {\sigma}^2}{ {\rho_L}^2},
\end{equation}
which completes the proof.
\end{proof}

\section{Appendix B: Experimental Details}

Our experiments were conducted on a workstation equipped with an Intel(R) Core(TM) i9-14900K CPU and an NVIDIA GeForce RTX 4090 GPU. The system uses CUDA Version 12.6 to support GPU-accelerated computations.

\subsection{Wireless MIMO Blockage Detection }
\begin{figure*}[htb]
  \centering
  \includegraphics[width=0.7\linewidth]{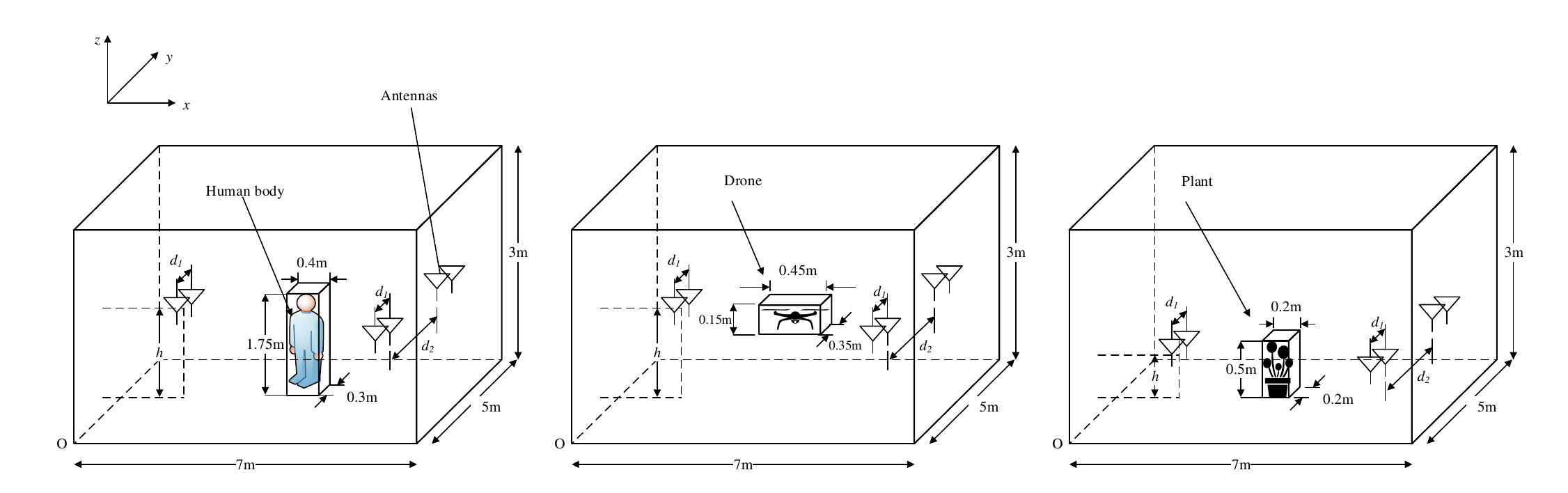}
  \caption{Antenna arrangement for blockage caused by the human body, UAV, and plant.}
  \label{Environment}
\end{figure*}

To evaluate the impact of various occlusions on MIMO signal propagation, we consider realistic blockage scenarios including human body, UAV interference, and vegetation (e.g., branches). These reflect common challenges such as indoor human blockage, drone interference, and natural obstacles. Our experiments use a \(2 \times 4\) MIMO system (2 transmitters, 4 receivers), with the antenna arrangement and environment shown in \cref{Environment}. The room size is \(7 \times 5 \times 3\, m^3\), with transmitting and receiving antennas placed at \(X=1\,m\) and \(6\,m\) respectively, typical for MIMO setups. We set \(d_1=0.1\,m\), \(d_2=2\,m\), and operate at 2.4 GHz with 10 MHz bandwidth. For human and drone blockage, antenna height is \(h=1\,m\). The human body and drone are modeled as cubes sized \(0.4 \times 0.3 \times 1.75\, m^3\) and \(0.45 \times 0.35 \times 0.15\, m^3\), respectively.

\subsection{Detection Time Comparison with GPU and CPU}

To evaluate the detection efficiency under different system scales, we compare the detection time of our GPU-based algorithm with that of the standard MIQP-based solver with gurobipy(Gurobi version: 12.0.2). The system scale is characterized by varying the dimension of the observation vector \( \mathbf{x} \). For each scale, we generate a corresponding system matrix \( \mathbf{H} \) , where each entry in a row is independently and randomly sampled from the set \(\{1, 0, -1\}\). $ \bar{\bm{h}_i}$ represents the estimated observation of $\bm{h}_i$, generated by introducing a random perturbation within the range of 0.1. The polyhedron uncertainty of $\mathbf{H}$ is set to be $ \bm{h}_i \in \mathcal{S}_i$ ,where is 

\[
\mathcal{S}_i = \left\{ \mathbf{h_i} \in \mathbb{R}^n \;\middle|\;
\begin{bmatrix}
\mathbf{I} \\
-\mathbf{I}
\end{bmatrix}
\mathbf{h_i}
\leq
\begin{bmatrix}
\bar{\mathbf{h}_i} + \sigma_0 \\
- \bar{\mathbf{h}_i} + \sigma_0
\end{bmatrix}
\right\}
\]
$\sigma_o$ is set as 0.1.
The $g_i(\cdot)$ in \eqref{eq:lagrangian} are
\begin{equation}
    \begin{aligned}
    g_1 &= \left\| 
    \begin{bmatrix}
    \mu^{+} - \mu^{-} \\
    \frac{\phi - u}{2}
    \end{bmatrix} 
    \right\| - \frac{\phi + u}{2} \leq 0 \\
    g_2 &= \boldsymbol{p}^T_i \boldsymbol{d}_i - \varepsilon_i \leq 0 \\
    g_3 &= \rho_L u - (\mu^{+} + \mu^{-}) \leq 0 \\
    g_4 &= (\mu^{+} + \mu^{-}) - \rho_U u \leq 0 \\
    g_5 &= \mu^{+} + \rho_U b - \rho_U \leq 0 \\
    g_6 &= \mu^{-} - \rho_U b \leq 0
\end{aligned}
\end{equation}
The GPU-based algorithm based on a neural unrolling architecture is set as \ref{fig:neural_unfolding} where number of unrolling layers $K$ are set as 10. The stopping criterion for each unrolling run in \cref{RoS-Guard} is defined as
$\epsilon < 0.01$,
where \(\epsilon\) denotes the relative change in variable updates between iterations,
with an additional safeguard of a maximum iteration 50 limit to ensure termination.
For variables with constraints, during network initialization, they are randomly initialized within the constraint bounds. The data and attack injection setting follow the Data I. Each experiment is repeated 50 times to account for randomness in initialization and input data. The reported detection times are averaged across all runs to ensure statistical reliability.

\subsection{Smart Grid Attack Injection Detection}
\label{Ad_E}
Data \uppercase\expandafter{\romannumeral1} is generated according to the IEEE 14-bus system. The IEEE 14-bus system as well as the
system matrix can be divided into four sub-regions.
The value of the entire system matrix can be found in \cite{methodin14}. 
For simplicity, we present only the matrix corresponding to the $4^{th}$ region.
\begin{equation}
\mathbf{H} = \begin{bmatrix}
-1 & 3 & -1 \\
0 & -1 & 0 \\
0 & 0 & -1 \\
0 & -1 & 1 \\
0 & -1 & 2 
\end{bmatrix}.
\end{equation}
In our experiment, we assume that the system matrix can be accurately determined in all areas except for the $4^{th}$ region.
The polyhedron uncertainty of $\mathbf{H}$ is set to be $ \mathbf{D}_{i} \mathbf{h}_{i}  \leq  \mathbf{c}_{i},~\forall i, $
where $\mathbf{h}_{i}$ represents the $i$-th column of $\mathbf{H}$. 
$\mathbf{D}_{1}$, $\mathbf{D}_{2}$ and $\mathbf{D}_{3}$ are set to be 
$$ \mathbf{D}_{1} = \mathbf{D}_{2} = \mathbf{D}_{3} = \begin{bmatrix}~  \mathbf{I} ~ \\ ~ \mathbf{-I} ~~  \end{bmatrix},$$ 
where  $\mathbf{I}$  represents the $5 \times 5$ identity matrix.  $\mathbf{c}_{1} $, $\mathbf{c}_{2}$ and $\mathbf{c}_{3}$ are set to be
\begin{equation}
\begin{aligned}
& \mathbf{c}_{1} = [ -0.5,0.5,0.5,0.5,0.5,1.5,0.5,0.5,0.5,0.5]^{T}, \\
& \mathbf{c}_{2} = [ 3.5,-0.5,0.5,-0.5,-0.5,-2.5,1.5,0.5,1.5,1.5]^{T}, \\
& \mathbf{c}_{3} = [ -0.5,0.5,-0.5,1.5,2.5,1.5,0.5,1.5,-0.5,-1.5]^{T}. 
\end{aligned}
\end{equation}
Regarding the ellipsoid uncertainty set, the uncertainty associated with each $\mathbf{h}_i$ can be described as follows:
\begin{equation}
\nonumber
\mathbf{h}_i \in \{\overline{ \mathbf{h}}_i +\mathbf{u}\mid  \|\mathbf{u}\|_2\leq 0.36\}, \quad \forall i,
\end{equation}
where $\overline{ \mathbf{h}}_i $ are set to be
\begin{equation}
\begin{aligned}
& \overline{ \mathbf{h}}_1  = [ -1.0,0.1,0.3,-0.2,0.0]^{T}, \\
& \overline{ \mathbf{h}}_2  = [ 3.0,-0.7,0.2,-1.3,-0.9]^{T}, \\
& \overline{ \mathbf{h}}_3  = [ -1.1,0.2,-0.6,0.7,2.0]^{T}. 
\nonumber
\end{aligned}
\end{equation}
In D-norm, it is assumed that the vector has at most $\kappa$ uncertain components, and each component falls within the error interval determined by $\hat{u}$ \cite{yang2014distributed}.
For all the $\mathbf{h}_{i}$, the parameter $\kappa$ and $\hat{u}$ are set to be $4$ and $0.5$, respectively.

\end{document}